\documentclass{article}

\usepackage[final, nonatbib]{neurips_2021}

\usepackage[utf8]{inputenc} 
\usepackage[T1]{fontenc}    
\usepackage{hyperref}       
\usepackage{url}            
\usepackage{booktabs}       
\usepackage{amsfonts}       
\usepackage{nicefrac}       
\usepackage{microtype}      
\usepackage{xcolor}         
\usepackage[numbers]{natbib}
\usepackage{kky}
\usepackage{multibib}
\usepackage{subcaption}
\usepackage{caption}
\newcites{online}{References}
\usepackage[symbol]{footmisc}

\usepackage{wrapfig}
\usepackage{algorithmic}
\usepackage{algorithm}

\usepackage{scrextend}

\usepackage[textwidth=\marginparwidth, textsize=tiny]{todonotes}

\title{Shared Independent Component Analysis for Multi-Subject Neuroimaging}

\author{%
  Hugo Richard \\
  Inria\\
  Université Paris-Saclay\\
  Palaiseau, France\\
  \And
  Pierre Ablin\\
  DMA\\
  CNRS and ENS \\
  Paris, France\\
  \And
  Bertrand Thirion \\
  Inria\\
  Université Paris-Saclay\\
  Palaiseau, France\\
  \And
  Alexandre Gramfort \\
  Inria\\
  Université Paris-Saclay\\
  Palaiseau, France\\
  \And
  Aapo Hyvärinen \\
  Department of Computer Science\\
  University of Helsinki\\
  Helsinki, Finland\\
}

\begin{document}

\maketitle

\begin{abstract}
  We consider shared response modeling, a multi-view learning problem where one
  wants to identify common components from multiple datasets or views.
  We introduce Shared Independent Component Analysis (ShICA) that models each view as a linear transform of shared independent components contaminated by additive Gaussian noise.
  We show that this model is identifiable if the components are either non-Gaussian or have enough diversity in noise variances. We then show that in some cases multi-set canonical correlation analysis can recover the correct unmixing matrices, but that even a small amount of sampling noise makes Multiset CCA fail.
  To solve this problem, we propose to use joint diagonalization after Multiset CCA, leading to a new approach called ShICA-J. We show via simulations that ShICA-J leads to improved results while being very fast to fit.
  While ShICA-J is based on second-order statistics, we further propose to leverage non-Gaussianity of the components using a maximum-likelihood method, ShICA-ML, that is both more accurate and more costly.
  Further, ShICA comes with a principled method for shared components estimation. Finally, we provide empirical evidence on fMRI and MEG datasets that ShICA yields more accurate estimation of the components than alternatives. 
\end{abstract}
\section{Introduction}
In many data science problems, data are available through different views. Generally, the views represent different measurement modalities such as audio and video, or the same text that may be available in different languages. Our main interest here is neuroimaging where recordings are made from multiple subjects. In particular, it is of interest to find common patterns or responses that are shared between subjects when they receive the same stimulation or perform the same cognitive task \citep{chen2015reduced,richard2020modeling}. 

A popular line of work to perform such shared response modeling is group Independent Component Analysis (ICA) methods. The fastest methods~\cite{calhoun2001method, varoquaux2009canica} are among the most popular, yet they are not grounded on principled probabilistic models for the multiview setting. 
More principled approaches exist~\cite{richard2020modeling, guo2008unified}, but they do not model subject-specific deviations from the shared response. However, such deviations are expected in most neuroimaging settings, as the magnitude of the response may differ from subject to subject \cite{penny2007random}, as may any noise due to heartbeats, respiratory artefacts or head movements~\cite{liu2016noise}.
Furthermore, most GroupICA methods are typically unable to separate components whose density is close to a Gaussian.

Independent vector analysis (IVA)~\cite{lee2008independent, anderson2011joint} is a powerful framework where components are independent within views but each component of a given view can depend on the corresponding component in other views. 
%
However, current implementations such as IVA-L~\cite{lee2008independent},
IVA-G~\cite{anderson2011joint}, IVA-L-SOS~\cite{bhinge2019extraction}, IVA-GGD~\cite{anderson2014independent} or
IVA with Kotz distribution~\cite{anderson2013independent} estimate only the
view-specific components, and do not model or extract a shared response which is
the main focus in this work.

%
On the other hand, the shared response model~\cite{chen2015reduced} is a popular approach to perform shared response modeling, yet it imposes orthogonality constrains that are restrictive and not biologically plausible.

In this work we introduce Shared ICA (ShICA), where each view is modeled as a linear transform of shared independent components contaminated by additive Gaussian noise. ShICA allows the principled extraction of the shared components (or responses) in addition to view-specific components. 
Since it is based on a statistically sound noise model, it enables optimal inference (minimum mean square error, MMSE) of the shared responses.

Let us note that ShICA is no longer the method of choice when the concept of common response is either not useful or not applicable. 
Nevertheless, we believe that the ability to extract a common response is an important feature in most contexts because it highlights a stereotypical brain response to a stimulus. Moreover, finding commonality between subjects reduces often unwanted inter-subject variability.

The paper is organized as follows.
We first analyse the theoretical properties of the ShICA model, before providing inference algorithms.
We exhibit necessary and sufficient conditions for the ShICA model to be identifiable (previous work only shows local identifiability~\cite{anderson2014independent}), in the presence of Gaussian or non-Gaussian components. 
%
We then use Multiset CCA to fit the model when all the components are assumed to
be Gaussian. We exhibit necessary and sufficient conditions for Multiset CCA to
be able to recover the unmixing matrices (previous work only gives sufficient
conditions~\cite{li2009joint}). In addition, we provide instances of the problem where Multiset CCA cannot recover the mixing matrices while the model is identifiable.
We next point out a practical problem : even a small sampling noise
can lead to large error in the estimation of unmixing matrices when Multiset CCA is used. To
address this issue and recover the correct unmixing matrices, we propose to
apply joint diagonalization to the result of Multiset CCA yielding a new method
called ShICA-J.
%
We further introduce ShICA-ML, a maximum likelihood estimator of ShICA that models non-Gaussian components using a Gaussian mixture model. 
While ShICA-ML yields more accurate components, ShICA-J is significantly faster and offers a great initialization to ShICA-ML.
Experiments on fMRI and MEG data demonstrate that the method outperforms existing GroupICA and IVA methods.

\section{Shared ICA (ShICA): an identifiable multi-view model}
\paragraph{Notation} We write vectors in bold letter $\vb$ and scalars in lower case $a$. Upper case letters $M$ are used to denote
matrices. We denote $|M|$ the absolute value of the determinant of $M$. $\xb \sim \Ncal(\mub, \Sigma)$ means that $\xb \in \mathbb{R}^k$ follows
a multivariate normal distribution of mean $\mub \in \mathbb{R}^k$ and
covariance $\Sigma \in \mathbb{R}^{k \times k}$. The $j, j$ entry of a diagonal matrix $\Sigma_i$ is denoted $\Sigma_{ij}$, the $j$ entry of $\yb_i$ is denoted $y_{ij}$. Lastly, $\delta$ is the Kronecker delta.

\paragraph{Model Definition} In the following, $\xb_1, \dots ,\xb_m \in \bbR^p$ denote the $m$ observed random vectors obtained from the $m$ different views. We posit the following generative model, called Shared ICA (ShICA): for $i= 1\dots m$
\begin{equation}
  \label{eq:model}
   \xb_i = A_i(\sbb + \nb_i)
\end{equation}
where $\sbb \in \mathbb{R}^{p}$ contains the latent variables called \emph{shared components}, $A_1,\dots, A_m\in\bbR^{p\times p}$ are the invertible mixing matrices, and $\nb_i \in
\mathbb{R}^{p}$ are \emph{individual noises}. The individual noises model both the deviations of a view  from the mean ---i.e.\ individual differences--- and measurement noise. Importantly, we explicitly model both the shared components and the individual differences in a probabilistic framework to enable an optimal inference of the parameters and the responses.

We assume that the shared components are statistically independent, and that the individual noises are Gaussian and independent from the shared components:
$p(\sbb) = \prod_{j=1}^p p(s_j)$ and $\nb_i \sim\mathcal{N}(0, \Sigma_i)$, where the matrices $\Sigma_i$ are assumed diagonal and positive. Without loss of generality, components are assumed to have unit variance $\bbE[\sbb \sbb^{\top}] = I_p$. We further assume that there are at least 3 views: $m \geq 3$. 

In contrast to almost all existing works, we assume that some components (possibly all of them) may be Gaussian, and denote $\mathcal{G}$ the set of Gaussian components: $\sbb_j \sim \mathcal{N}(0, 1)$ for $j \in \mathcal{G}$. The other components are non-Gaussian: for $j\notin \mathcal{G}$, $\sbb_j$ is non-Gaussian.

\paragraph{Identifiability} The parameters of the model are $\Theta = (A_1, \dots, A_m, \Sigma_1, \dots, \Sigma_m)$. We are interested in the identifiability of this model: given observations $\xb_1,\dots, \xb_m$ generated with parameters $\Theta$, are there some other $\Theta'$ that may generate the same observations?
Let us consider the following assumption that requires that the individual noises for Gaussian components are sufficiently diverse:
\begin{assumption}[Noise diversity in Gaussian components]
\label{ass:diversity}
For all $j, j' \in \mathcal{G}, j \neq j'$, the sequences $(\Sigma_{ij})_{i=1 \dots m}$ and $(\Sigma_{ij'})_{i=1 \dots m}$ are different where $\Sigma_{ij}$ is the $j, j$ entry of $\Sigma_i$
\end{assumption}

It is readily seen that there is one trivial set of indeterminacies in the problem: if $P \in \mathbb{R}^{p \times p}$ is a sign and permutation matrix (i.e. a matrix which has one $\pm 1$ coefficient on each row and column, and $0$'s elsewhere) the parameters $(A_1 P, \dots, A_m P, P^{\top}\Sigma_1 P, \dots, P^{\top} \Sigma_m P)$ also generate $\xb_1,\dots, \xb_m$. The following theorem shows that under the above assumption, these are the only indeterminacies of the problem.

\begin{theorem}[Identifiability]
\label{thm:identif}
We make Assumption~\ref{ass:diversity}. We let $\Theta'=(A_1', \dots, A_m', \Sigma_1', \dots,\Sigma_m')$ another set of parameters, and assume that they also generate $\xb_1,\dots, \xb_m$. Then, there exists a sign and permutation matrix $P$ such that for all $i$, $A_i'=A_iP$, and $\Sigma_i'= P^{\top} \Sigma_i P$.
\end{theorem}
The proof is in Appendix~\ref{proof:identif}. Identifiability in the Gaussian case is a consequence of the identifiability results in~\cite{via2011joint} and in the general case, local identifiability results can be derived from the work of ~\cite{anderson2014independent}. 
However local identifiability only shows that for a given set of parameters there exists a neighborhood in which no other set of parameters can generate the same observations~\cite{rothenberg1971identification}. In contrast, the proof of Theorem~\ref{thm:identif} shows global identifiability.

Theorem~\ref{thm:identif} shows that the task of recovering the parameters from the observations is a well-posed problem, under the sufficient condition of Assumption~\ref{ass:diversity}.  We also note that Assumption~\ref{ass:diversity} is necessary for identifiability. For instance, if $j$ and $j'$ are two Gaussian components such that $\Sigma_{ij} = \Sigma_{ij'}$ for all $i$, then a global rotation of the components $j, j'$ yields the same covariance matrices. The current work assumes $m \geq 3$, in appendix~\ref{app:identifiability} we give an identifiability result for $m=2$, under stronger conditions.

\section{Estimation of components with noise diversity via joint-diagonalization}

We now consider the computational problem of efficient parameter inference. This section considers components with noise diversity, while the next section deals with non-Gaussian components.

\subsection{Parameter estimation with Multiset CCA}
If we assume that the components are all Gaussian, 
the covariance of the observations given by
$C_{ij}=  \bbE[\xb_i\xb_j^\top] = A_i(I_p + \delta_{ij}\Sigma_i)A_j^{\top}\enspace
$ are sufficient statistics and methods using only second order information, like Multiset CCA, are candidates to estimate the parameters of the model.
Consider the
matrix $\mathcal{C} \in \bbR^{pm \times pm}$ containing $m \times m$ blocks of size $p
\times p$
such that the block $i,j$ is given by $C_{ij}$. Consider the matrix $\mathcal{D}$ identical to $\mathcal{C}$ excepts that the non-diagonal blocks are filled with zeros:
\begin{equation}
  \mathcal{C} = \begin{bmatrix}
    C_{11} & \dots & C_{1m}\\
    \vdots & \ddots & \vdots \\
    C_{m1} &\dots & C_{mm} 
  \end{bmatrix}
  ,\enspace
  \mathcal{D} = \begin{bmatrix}
    C_{11} & \dots & 0\\
    \vdots & \ddots & \vdots \\
    0 &\dots & C_{mm} 
  \end{bmatrix}. 
\end{equation} 
Generalized CCA consists of the following generalized eigenvalue problem:
\begin{equation}
\label{eq:eigv}
    \mathcal{C} \ub = \lambda \mathcal{D}\ub,\enspace \lambda > 0,\enspace \ub\in\bbR^{pm} \enspace .
\end{equation}
  
Consider the matrix $U = [\ub^1, \dots, \ub^p] \in \mathbb{R}^{mp \times p}$ formed by concatenating the $p$ leading eigenvectors of the previous problem ranked in decreasing eigenvalue order. Then, consider $U$ to be formed of $m$ blocks of size $p \times p$ stacked vertically and define $(W^i)^{\top}$ to be the $i$-th block. These $m$ matrices are the output of Multiset CCA. We also denote $\lambda_1 \geq \dots \geq \lambda_p$ the $p$ leading eigenvalues of the problem.

An application of the results of \cite{li2009joint} shows that Multiset CCA recovers the mixing matrices of ShICA under some assumptions.
\begin{proposition}[Sufficient condition for solving ShICA via Multiset CCA~\cite{li2009joint}]
Let $r_{ijk} = (1 + \Sigma_{ik})^{-\frac12} (1 + \Sigma_{jk})^{-\frac12}$.
Assume that $(r_{ijk})_k$ is non-increasing. Assume that the maximum eigenvalue $\nu_k$ of matrix $R^{(k)}$ of general element $(r_{ijk})_{ij}$ is such that  $\nu_k = \lambda_k$ 
.
Assume that $\lambda_1 \dots \lambda_p$ are distinct.
Then, there exists scale matrices $\Gamma_i$ such that $W_i = 
\Gamma_i A_i^{-1}$ for all $i$.
\end{proposition}
This proposition gives a sufficient condition for solving ShICA with Multiset CCA. It needs a particular structure for the noise covariances as well as specific ordering for the eigenvalues. The next theorem shows that we only need $\lambda_1 \dots \lambda_p$ to be distinct for Multiset CCA to solve ShICA:
\begin{assumption}[Unique eigenvalues]
  \label{ass:uniqueeig}
$\lambda_1 \dots \lambda_p$ are distinct.
\end{assumption}
\begin{theorem}
  \label{th:eig}
  We only make
  Assumption~\ref{ass:uniqueeig}. Then, there exists a permutation matrix $P$ and scale matrices $\Gamma_i$ such that $W_i = P\Gamma_i A_i^{-1}$ for all $i$.
\end{theorem}
The proof is in Appendix~\ref{proof:eig}. This theorem means that solving the generalized eigenvalue problem~\eqref{eq:eigv} allows to recover the mixing matrices up to a scaling and permutation: this form of generalized CCA recovers the parameters of the statistical model.
Note that Assumption~\ref{ass:uniqueeig} is also a necessary condition. Indeed, if two eigenvalues are identical, the eigenvalue problem is not uniquely determined.

We have two different Assumptions, \ref{ass:diversity} and \ref{ass:uniqueeig}, the first of which guarantees theoretical identifiability as per Theorem~\ref{thm:identif} and the second guarantees consistent estimation by Multiset CCA as per Theorem~\ref{th:eig}. Next we will discuss their connections, and show some limitations of the Multiset CCA approach. To begin with, we have the following result about the eigenvalues of the problem~\eqref{eq:eigv} and the $\Sigma_{ij}$.
\begin{proposition}
  \label{prop:eigvals_from_noise}
  For $j\leq p$, let $\lambda_j$ the largest solution of $ \sum_{i=1}^m\frac{1}{\lambda_j(1 + \Sigma_{ij}) -\Sigma_{ij}}=1$. Then, $\lambda_1, \dots, \lambda_p$ are the $p$ largest eigenvalues of problem~\eqref{eq:eigv}.
\end{proposition}
It is easy to see that we then have $\lambda_1, \dots, \lambda_p$ greater than $1$, while the remaining eigenvalues are lower than $1$.
From this proposition, two things appear clearly. First, Assumption~\ref{ass:uniqueeig} implies Assumption~\ref{ass:diversity}.
Indeed, if the $\lambda_j$'s are distinct, then the sequences $(\Sigma_{ij})_i$ must also be different from the previous proposition.
This is expected as from Theorem~\ref{th:eig}, Assumption~\ref{ass:uniqueeig} implies identifiability, which in turn implies Assumption~\ref{ass:diversity}.

Prop.~\ref{prop:eigvals_from_noise} also allows us to derive cases where Assumption~\ref{ass:diversity} holds but not Assumption~\ref{ass:uniqueeig}. The following Proposition gives a simple case where the model is identifiable but it cannot be solved using Multiset CCA:
\begin{proposition}
\label{counter}
Assume that for two integers $j, j'$, the sequence $(\Sigma_{ij})_i$ is a permutation of $(\Sigma_{ij'})_i$, i.e. that there exists a permutation of $\{1,\dots, p\}$, $\pi$, such that for all $i$, $\Sigma_{ij} = \Sigma_{\pi(i)j'}$.  Then, $\lambda_j = \lambda_{j'}$.
\end{proposition}
In this setting, Assumption~\ref{ass:diversity} holds so ShICA is identifiable, while Assumption~\ref{ass:uniqueeig} does not hold, so Multiset CCA cannot recover the unmixing matrices.

\subsection{Sampling noise and improved estimation with joint diagonalization} \label{sec:samplingnoise}

The consistency theory for Multiset CCA developed above is conducted under the assumption that the
covariances $C_{ij}$ are the true covariances of the model, and not
approximations obtained from observed samples. In practice, however, a serious limitation of Multiset CCA is that even a slight error of estimation on the covariances, due to ``sampling noise'', can yield a large error in the estimation of the unmixing matrices, as will be shown next.

We begin with an empirical illustration. We take $m=3$, $p=2$, and $\Sigma_i$ such that $\lambda_1 = 2 + \varepsilon$ and $\lambda_2 =2$ for $\varepsilon > 0$.
In this way, we can control the \emph{eigen-gap} of the problem, $\varepsilon$.
We take $W_i$ the outputs of Multiset CCA applied to the true covariances $C_{ij}$.
Then, we generate a perturbation $\Delta = \delta \cdot S$, where $S$ is a random positive symmetric $pm \times pm$ matrix of norm $1$, and $\delta >0$ controls the scale of the perturbation. 
We take $\Delta_{ij}$ the $p\times p$ block of $\Delta$ in position $(i, j)$, and $\tilde{W}_i$ the output of Multiset CCA applied to the covariances $C_{ij} + \Delta_{ij}$.
We finally compute the sum of the Amari distance between the $W_i$ and $\tilde{W}_i$: the Amari distance measures how close the two matrices are, up to scale and permutation~\cite{amari1996new}.
\begin{wrapfigure}{r}{.4\textwidth}
  \centering
  \includegraphics[width=.99\linewidth]{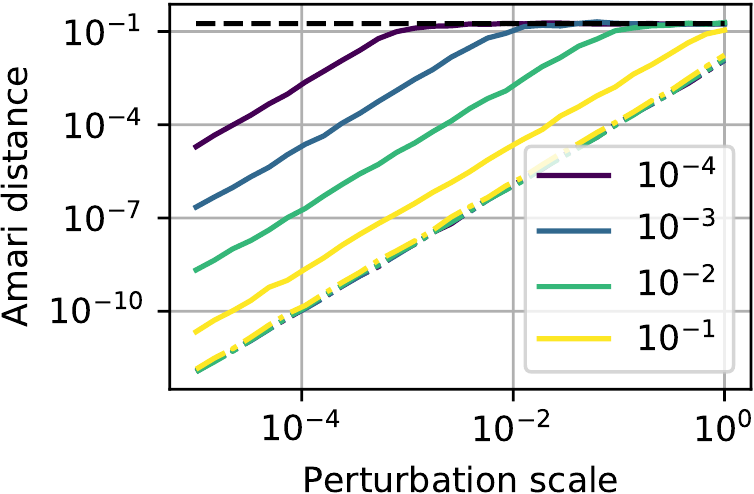}
  \caption{Amari distance between true mixing matrices and estimates of Multiset
    CCA when covariances are perturbed. Different solid curves correspond to different
    eigen-gaps. The black dotted line shows the chance level. When the gap is small, a small perturbation can lead to complete mixing. Joint-diagonalization (colored dotted lines) fixes the problem.}
  \label{fig:cca_gap}
  \end{wrapfigure}
Fig~\ref{fig:cca_gap} displays the median Amari distance over 100 random repetitions, as the perturbation scale $\delta$ increases. The different curves correspond to different values of the eigen-gap $\varepsilon$. We see clearly that the robustness of Multiset CCA critically depends on the eigen-gap, and when it is small, even a small perturbation of the input (due, for instance, to sampling noise) leads to large estimation errors.

This problem is very general and well studied~\cite{stewart1973error}: the mapping from matrices to (generalized) eigenvectors is highly non-smooth.
However, the gist of our method is that the \emph{span} of the leading $p$ eigenvectors is smooth, as long as there is a large enough gap between  $\lambda_p$ and $\lambda_{p+1}$.
For our specific problem we have the following bounds, derived from Prop.~\ref{prop:eigvals_from_noise}.
\begin{proposition}
  We let $\sigma_{\max} = \max_{ij}\Sigma_{ij}$ and $\sigma_{\min} = \min_{ij}\Sigma_{ij}$. Then, $\lambda_p \geq 1 + \frac{m-1}{1+\sigma_{\max}}$, while $\lambda_{p+1}\leq 1 - \frac{1}{1 + \sigma_{min}}$.
\end{proposition}
As a consequence, we have $\lambda_{p} -\lambda_{p+1} \geq \frac{m-1}{1+\sigma_{\max}} + \frac{1}{1+ \sigma_{\min}}\geq \frac m{1+ \sigma_{\max}}$: the gap between these eigenvalues increases with $m$, and decreases with the noise power.

\begin{wrapfigure}{l}{.45\textwidth}
\begin{minipage}{.45\textwidth}
  \begin{algorithm}[H]
  \caption{ShICA-J}
  \label{algo:shicaj}
    \begin{algorithmic}
       \STATE {\bfseries Input :} Covariances $\tilde{C}_{ij} = \bbE[\xb_i\xb_j^{\top}]$
       \STATE $(\tilde{W}_i)_i \leftarrow \mathrm{MultisetCCA}((\tilde{C}_{ij})_{ij})$
       \STATE  $Q \leftarrow \mathrm{JointDiag}((\tilde{W}_i\tilde{C}_{ii}\tilde{W}_i^{\top})_i)$
       \STATE $\Gamma_{ij} \leftarrow Q\tilde{W}_i\tilde{C}_{ij}W_j^\top Q^\top$
       \STATE $(\Phi_i)_i \leftarrow \mathrm{Scaling}((\Gamma_{ij})_{ij})$
         \STATE \textbf{Return : } Unmixing matrices $(\Phi_iQ\tilde{W}_i)_i$.
        \end{algorithmic}
  \end{algorithm}
\end{minipage}
\end{wrapfigure}
In this setting, when the magnitude of the perturbation $\Delta$ is smaller than $\lambda_{p}-\lambda_{p+1}$, ~\cite{stewart1973error} indicates that $\mathrm{Span}([W_1, \dots, W_m]^{\top})\simeq \mathrm{Span}([\tilde{W}_1,\dots, \tilde{W}_m]^\top)$, where $[W_1, \dots, W_m]^{\top}\in\bbR^{pm\times p}$ is the vertical concatenation of the $W_i$'s.
In turn, this shows that there exists a matrix $Q\in\bbR^{p\times p}$ such that
\begin{equation}
    \label{eq:justif_jd}
    W_i \simeq Q\tilde{W}_i\enspace \text{for all} \enspace i.
\end{equation}
We propose to use joint-diagonalization to recover the matrix $Q$. Given the $\tilde{W}_i$'s, we consider the set of symmetric matrices $\tilde{K}_i = \tilde{W}_i\tilde{C}_{ii}\tilde{W}_i^{\top}$, where $\tilde{C}_{ii}$ is the contaminated covariance of $\xb_i$. Following Eq.~\eqref{eq:justif_jd}, we have $Q\tilde{K}_iQ^{\top} = W_i \tilde{C}_{ii}W_i^{\top}$, and using Theorem~\ref{th:eig}, we have $Q\tilde{K}_iQ^{\top} = P\Gamma_i A_i^{-1}\tilde{C}_{ii}A_i^{-\top}\Gamma_iP^{\top}$. Since $\tilde{C}_{ii}$ is close to $C_{ii} = A_i (I_p + \Sigma_i)A_i^\top$, the matrix $P\Gamma_i A_i^{-1}\tilde{C}_{ii}A_i^{-\top}\Gamma_iP^{\top}$ is almost diagonal.
In other words, the matrix $Q$ is an approximate diagonalizer of the $\tilde{K}_i$'s, and we approximate $Q$ by joint-diagonalization of the $\tilde{K}_i$'s. In Fig~\ref{fig:cca_gap}, we see that this procedure mitigates the problems of multiset-CCA, and gets uniformly better performance regardless of the eigen-gap.
In practice, we use a fast joint-diagonalization algorithm~\cite{ablin2018beyond} to minimize a joint-diagonalization criterion for positive symmetric matrices~\cite{pham2001joint}. The estimated unmixing matrices $U_i = Q\tilde{W}_i$ correspond to the true unmixing matrices only up to some scaling which may be different from subject to subject: the information that the components are of unit variance is lost. As a consequence, naive averaging of the recovered components may lead to inconsistant estimation. We now describe a procedure to recover the correct scale of the individual components across subjects.

\textbf{Scale estimation}
We form the matrices $\Gamma_{ij} = U_i\tilde{C}_{ij}U_j^\top$. In order to estimate the scalings, we solve $
\min_{(\Phi_i)} \sum_{i\neq j} \| \Phi_i \diag(\Gamma_{ij}) \Phi_j - I_p \|_F^2$
where the $\Phi_i$ are diagonal matrices.
This function is readily minimized with respect to one of the $\Phi_i$ by the formula
$\Phi_i = \frac{\sum_{j \neq i} \Phi_j \diag(Y_{ij})}{\sum_{j \neq i} \Phi_j^2 \diag(Y_{ij})^2}$ (derivations in Appendix~\ref{app:fixedpoint}). We then iterate the previous formula over $i$ until convergence.
The final estimates of the unmixing matrices are given by
$(\Phi_i U_i)_{i=1}^m$. The full procedure, called ShICA-J, is summarized in Algorithm~\ref{algo:shicaj}.

\subsection{Estimation of noise covariances}

In practice, it is important to estimate noise covariances $\Sigma_i$ in order to take advantage of the fact that some views are noisier than others. As it is well known in classical factor analysis, modelling noise variances allows the model to virtually discard variables, or subjects, that are particularly noisy. 

Using the ShICA model with Gaussian components, we derive an estimate for the noise covariances directly from maximum likelihood. We use an expectation-maximization (EM) algorithm, which is especially fast because noise updates are in closed-form. Following derivations given in appendix~\ref{conditional_density}, the sufficient statistics in the E-step are given by 
\begin{align}
\label{mmse1}
\EE[\sbb|\xb]= \left(\sum_{i=1}^m \Sigma_i^{-1}  + I \right)^{-1}  \sum_{i=1}^m \left(\Sigma_i^{-1} \yb_i \right)
     && \VV[\sbb|\xb]= (\sum_{i=1}^m \Sigma_i^{-1}  + I)^{-1}
\end{align}
Incorporating the M-step we get the following updates that only depend on the covariance matrices:
$
\Sigma_i \leftarrow \diag(\hat{C_{ii}} - 2 \VV[\sbb | \xb]  \sum_{j=1}^m \Sigma_j^{-1} \hat{C}_{ji}  + \VV[\sbb | \xb]  \sum_{j = 1}^m \sum_{l = 1}^m \left(\Sigma_j^{-1} \hat{C}_{jl} \Sigma_l^{-1} \right) \VV[\sbb | \xb] + \VV[\sbb | \xb])
$

\section{ShICA-ML: Maximum likelihood for non-Gaussian components}
ShICA-J only uses second order statistics. However, the ShICA model~\eqref{eq:model} allows for non-Gaussian components. We now propose an algorithm for fitting the ShICA model that combines covariance information with non-Gaussianity in the estimation to optimally separate both Gaussian and non-Gaussian components.
We estimate the parameters by maximum likelihood. Since most non-Gaussian
components in real data are super-Gaussian~\cite{delorme2012independent, calhoun2006unmixing}, we assume that the non-Gaussian components $\sbb$ have the super-Gaussian density \\ $p(s_j) = \frac12\left(\mathcal{N}( s_j; 0, \frac12) + \mathcal{N}( s_j; 0, \frac{3}{2})\right) \enspace.$

We propose to maximize the log-likelihood using a generalized
EM~\cite{neal1998view, dempster1977maximum}. Derivations are available in Appendix~\ref{app:emestep}. Like in the previous section, the E-step is in closed-form yielding the following sufficient statistics:
\begin{align}
\label{mmse2}
  \EE[s_j | \xb] = \frac{\sum_{\alpha \in \{\frac12, \frac32\}} \theta_{\alpha} \frac{\alpha \bar{y}_{j}}{\alpha + \bar{\Sigma_{j}}}}{\sum_{\alpha \in \{0.5, 1.5\}} \theta_{\alpha}} \enspace \text{ and } \enspace \VV[s_j | \xb] = \frac{\sum_{\alpha \in \{\frac12, \frac32\}} \theta_{\alpha} \frac{\bar{\Sigma_{j}}\alpha}{\alpha + \bar{\Sigma_{j}}}}{\sum_{\alpha \in \{0.5, 1.5\}} \theta_{\alpha}}  
\end{align}
    where $\theta_{\alpha} = \Ncal(\bar{y}_{j}; 0 , \bar{\Sigma}_{j} + \alpha)$, 
    $\bar{y}_j = \frac{\sum_i \Sigma_{ij}^{-1} y_{ij}}{ \sum_i
      \Sigma_{ij}^{-1}}$ and $\bar{\Sigma_{j}} = (\sum_i
    \Sigma_{ij}^{-1})^{-1}$ with $\yb_i = W_i \xb_i$.
Noise updates are in closed-form and given by:
$\Sigma_i \leftarrow  \diag((\yb_i - \EE[\sbb | \xb]) (\yb_i - \EE[\sbb | \xb])^{\top}+ \VV[\sbb | \xb])$.
However, no closed-form is available for the updates of unmixing matrices. We therefore perform quasi-Newton updates given by
$W_i \leftarrow (I - \rho (\widehat{\mathcal{H}^{W_i}})^{-1} \mathcal{G}^{W_i}) W_i$ where $\rho \in \mathbb{R}$ is chosen by backtracking line-search,
$\widehat{\mathcal{H}^{W_i}_{a, b, c, d}} =  \delta_{ad} \delta_{bc} +
\delta_{ac} \delta_{bd}\frac{(y_{ib})^2}{\Sigma_{ia}}$
is an approximation of the Hessian
of the negative complete likelihood and $\mathcal{G}^{W_i} = -I + (\Sigma_i)^{-1}(\yb_i - \mathbb{E}[\sbb|\xb])(\yb_i)^{\top}$ is the gradient.

We alternate between computing the statistics $\mathbb{E}[\sbb|\xb]$, 
$\mathbb{V}[\sbb|\xb]$ (E-step) and updates of parameters $\Sigma_i$ and $W_i$ for $i=1 \dots m$ (M-step). Let us highlight that our EM algorithm and in particular the E-step resembles the one used in~\cite{moulines1997maximum}. However because they assume noise on the sensors and not on the components, their formula for $\EE[\sbb| \xb]$ involves a sum with $2^p$ terms whereas we have only $2$ terms. The resulting method is called ShICA-ML.

\paragraph{Minimum mean square error estimates in ShICA}
In ShICA-J as well as in ShICA-ML, we have a closed-form for the expected components given the data $\EE[\sbb | \xb]$, shown in equation~\eqref{mmse1} and~\eqref{mmse2} respectively. This provides minimum mean square error estimates of the shared components, and is an important benefit of explicitly modelling shared components in a probabilistic framework.

\section{Related Work}
ShICA combines theory and methods coming from different branches of ``component analysis''. It can be viewed as a GroupICA method, as an extension of Multiset CCA, as an Independent Vector Analysis method or, crucially, as an extension of the shared response model. In the setting studied here, ShICA improves upon all existing methods.

\paragraph{GroupICA}
GroupICA methods extract independent components from multiple datasets. In its original form\cite{calhoun2001method}, views are concatenated and then a PCA is applied yielding reduced data on which ICA is applied. One can also reduce the data using Multiset CCA instead of PCA, giving a method called \emph{CanICA}~\cite{varoquaux2009canica}. Other works~\cite{Esposito05NI, Hyva11NI} apply ICA separately on the datasets and attempt to match the decompositions afterwards.
Although these works provide very fast methods, they do not rely on a well defined model like ShICA. 
Other GroupICA methods impose some structure on the mixing matrices such as the tensorial method of~\cite{beckmann2005tensorial} or the group tensor model in~\cite{guo2008unified} (which assumes identical mixing matrices up to a scaling) or \cite{svensen2002ica} (which assumes identical mixing matrices but different components). In ShICA the mixing matrices are only constrained to be invertible.
Lastly, maximum-likelihood based methods exist such as
\emph{MultiViewICA}~\cite{richard2020modeling} (MVICA) or the full model
of~\cite{guo2008unified}.
These methods are weaker than ShICA as they use the same noise covariance across views and lack a principled method for shared response inference.

\paragraph{Multiset CCA}
In its basic formulation, CCA identifies a shared space between two datasets.
The extension to more than two datasets is ambiguous, and many
different generalized CCA methods have been proposed. \cite{kettenring1971canonical} introduces 6 objective functions that reduce to CCA when $m=2$ and \cite{nielsen2002multiset} considered 4 different possible constrains leading to 24 different formulations of Multiset CCA. The formulation used in ShICA-J is refered to in~\cite{nielsen2002multiset} as SUMCORR with constraint 4 which is one of the fastest as it reduces to solving a generalized eigenvalue problem. The fact that CCA solves a well defined probabilistic model has first been studied in~\cite{bach2005probabilistic} where it is shown that CCA is identical to multiple battery factor analysis~\cite{browne1980factor} (restricted to 2 views). This latter formulation differs from our model in that the noise is added on the sensors and not on the components which makes the model unidentifiable. Identifiable variants and
generalizations can be obtained by imposing sparsity on the mixing matrices such as in~\cite{archambeau2008sparse, klami2014group, witten2009extensions} or non-negativity~\cite{DELEUS2011143}.
The work in~\cite{li2009joint} exhibits a set of sufficient (but not necessary) conditions under which a well defined model can be learnt by the formulation of Multiset CCA used in ShICA-J. The set of conditions we exhibit in this work are necessary and sufficient. We further emphasize that basic Multiset CCA provides a poor estimator as explained in Section~\ref{sec:samplingnoise}.

\paragraph{Independent vector analysis}
Independent vector analysis~\cite{lee2008independent} (IVA) models the data as a linear mixture of independent components $\xb_i = A_i \sbb_i$ where each component $s_{ij}$ of a given view $i$ can depend on the corresponding component in other views ($(s_{ij})_{i=1}^m$ are not independent).
Practical implementations of this very general idea assume a distribution for
$p((s_{ij})_{i=1}^m)$. In IVA-L~\cite{lee2008independent}, $p((s_{ij})_{i=1}^m)
\propto \exp(-\sqrt{\sum_i (s_{ij})^2})$ (so the variance of each component in
each view is assumed to be the same), in IVA-G~\cite{anderson2011joint} or
in~\cite{via2011maximum}, $p((s_{ij})_{i=1}^m) \sim \mathcal{N}(0, R_{ss})$
and~\cite{engberg2016independent} proposed a normal inverse-Gamma density.
Let us also mention IVA-L-SOS~\cite{bhinge2019extraction}, IVA-GGD~\cite{anderson2014independent} and
IVA with Kotz distribution~\cite{anderson2013independent} that assume a
non-Gaussian density general enough so that they can use both second and higher
order statistics to extract view-specific components.
The model of ShICA can be seen as an instance of IVA 
which specifically enables extraction of shared components from the subject specific components, unlike previous versions of IVA. In fact, ShICA comes with minimum mean square error estimates for the shared components
that is often the quantity of interest.
The IVA theory provides global identifiability conditions in the Gaussian case (IVA-G)~\cite{via2011joint} and local identifiability conditions in the general case~\cite{anderson2014independent} from which local identifiability conditions of ShICA could be derived. However, in this work, we provide global identifiability conditions for ShICA.
Lastly, IVA can be performed using joint diagonalization of cross covariances~\cite{li2011joint, congedo2012orthogonal} although multiple matrices have to be learnt and cross-covariances are not necessarily symmetric positive definite which makes the algorithm slower and less principled.

\paragraph{Shared response model}
ShICA extracts shared components from multiple datasets, which is also the goal
of the shared response model (SRM)~\cite{chen2015reduced}. The robust
SRM~\cite{turek2018capturing} also allows to capture subject specific noise.
However these models impose orthogonality constraints on the mixing matrices
while ShICA does not.
Deep variants of SRM exist such
as~\cite{chen2016convolutional} but while they release the orthogonality
constrain, they are not very easy to train or interpret and have many
hyper-parameters to tune. ShICA leverages ICA theory to provide a much more powerful model of shared responses.

\paragraph{Limitations}
The main limitation of this work is that the model cannot reduce the dimension inside each view : there are as many estimated sources as sensors. This might be problematic when the number of sensors is very high. In line with other methods, view-specific  dimension reduction has to be done by some external method, typically view-specific PCA. Using specialized methods for the estimation of covariances should also be of interest for ShICA-J, where it only relies on sample covariances. Finally, ShICA-ML uses a simple model of a super-Gaussian distribution, while modelling the non-gaussianities in more detail in ShICA-ML should improve the performance.

\section{Experiments}
Experiments used Nilearn~\cite{abraham2014machine} and MNE~\cite{gramfort2013meg} for fMRI and MEG data
processing respectively, as well as the scientific Python ecosystem:
Matplotlib~\cite{hunter2007matplotlib}, Scikit-learn~\cite{pedregosa2011scikit},
Numpy~\cite{harris2020array} and Scipy~\cite{2020SciPy-NMeth}. We use the Picard algorithm for non-Gaussian ICA~\cite{ablin2018faster}, and mvlearn for multi-view ICA~\cite{perry2020mvlearn}. The above libraries use open-source licenses. fMRI experiments used the following datasets: sherlock~\cite{chen2017shared}, forrest~\cite{hanke2014high} , raiders~\cite{ibc} and gallant~\cite{ibc}. The data we use do not contain offensive content or identifiable information and consent was obtained before data collection. Computations were run on a large server using up to 100 GB of RAM and 20 CPUs in parallel.
%
\paragraph{Separation performance}
\label{sec:rotation}
In the following synthetic experiments, data are generated according to model~\eqref{eq:model} with $p=4$ components and $m=5$ views and mixing matrices are generated by sampling coefficients from a standardized Gaussian.
Gaussian components are generated from a standardized Gaussian and their noise
has standard deviation $\Sigma_i^{\frac12}$ (obtained by sampling from a uniform
density between $0$ and $1$) while non-Gaussian components are generated from a
Laplace distribution and their noise standard deviations are equal. We study 3
cases where either all components are Gaussian, all components are non-Gaussian
or half of the components are Gaussian and half are non-Gaussian.
We vary the
number of samples $n$ between $10^2$ and $10^5$ and display in
Fig~\ref{exp:rotation} the mean Amari distance across subjects between the true unmixing
matrices and estimates of algorithms as a function of $n$.
The experiment is repeated $100$ times using different seeds. We report the median result and error bars represent the first and last deciles.
\begin{figure}
  \centering
  \includegraphics[width=\textwidth]{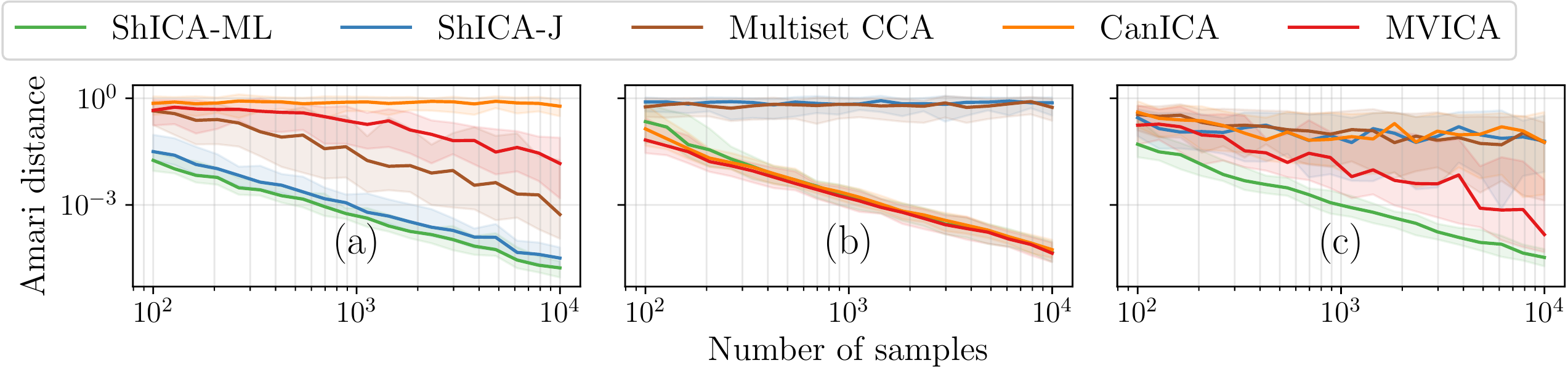}
  \caption{\textbf{Separation performance}: Algorithms are fit on data following model~\ref{eq:model} \textbf{(a)} Gaussian components with noise diversity \textbf{(b)} Non-Gaussian components without noise diversity \textbf{(c)} Half of the components are Gaussian with noise diversity, the other half is non-Gaussian without noise diversity. }
  \label{exp:rotation}
\end{figure}

When all components are Gaussian (Fig.~\ref{exp:rotation}~(a)), CanICA cannot
separate the components at all. In contrast ShICA-J, ShICA-ML, Multiset CCA and
MVICA are able to separate them, but Multiset CCA needs many more samples than
ShICA-J or ShICA-ML to reach a low amari distance, which shows that correcting for the rotation due to sampling noise improves the results. Looking at error bars, we also see that the performance of Multiset CCA varies quite a lot with the random seeds: this shows that depending on the sampling noise, the rotation can be very different from identity.
MVICA needs even more sample than Multiset CCA to reach a low amari distance but
still outperforms CanICA.

When none of the components are Gaussian (Fig.~\ref{exp:rotation}~(b)), only
CanICA, ShICA-ML and MVICA are able to separate the components, as other methods do not make use of non-Gaussianity.
Finally, in the hybrid case (Fig.~\ref{exp:rotation}~(c)), ShICA-ML is able to
separate the components as it can make use of both non-Gaussianity and noise
diversity. MVICA is a lot less reliable than ShICA-ML, it is uniformly worse and
error bars are very large showing that for some seeds it gives poor results.
CanICA, ShICA-J and MultisetCCA cannot separate the components at all.
Additional experiments illustrating the separation powers of algorithms are available in Appendix~\ref{app:separation}.

As we can see, MVICA can separate Gaussian components to some extent and therefore does not completely fail when Gaussian and non-Gaussian components are present. However MVICA is a lot less reliable than ShICA-ML: MVICA is uniformly worse than ShICA-ML and the error bars are very large showing that for some seeds it gives poor results.

\paragraph{Computation time}


\begin{figure}
  \centering
  \begin{subfigure}[b]{0.49\textwidth}
      \centering
      \includegraphics[width=\textwidth]{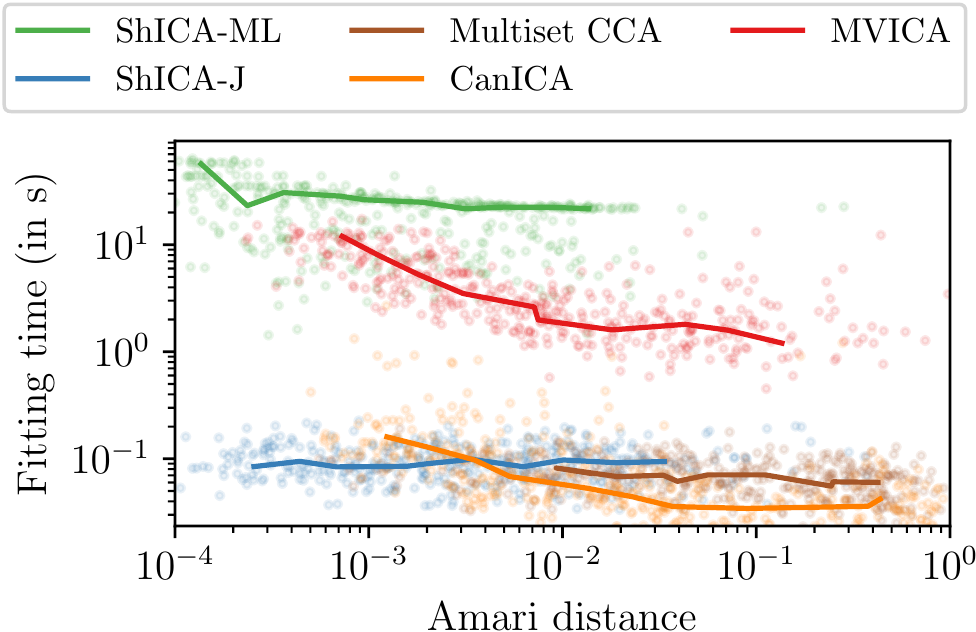}
      \caption{}
      \label{exp:syn_timings}
  \end{subfigure}
  \hfill
  \begin{subfigure}[b]{0.45\textwidth}
    \centering
    \includegraphics[width=\textwidth]{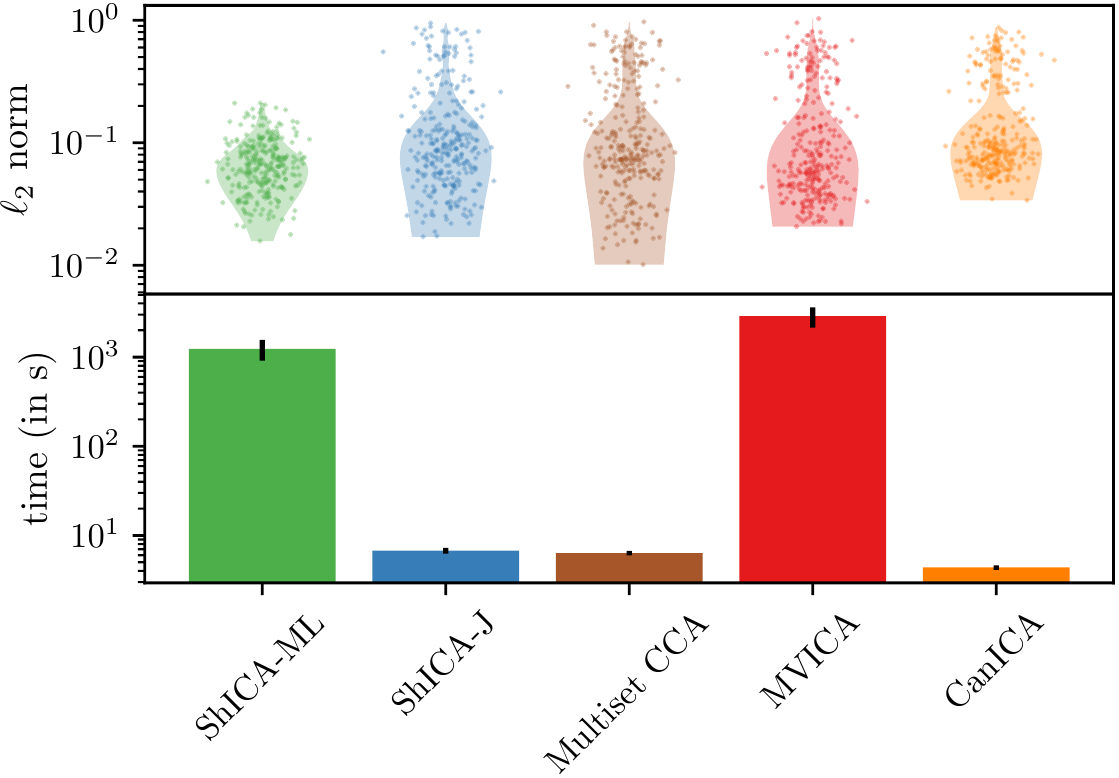}
    \caption{}
    \label{fig:eeg_intragroup_variability}
\end{subfigure}
     \caption{\textbf{Left: Computation time.} Algorithms are fit on data generated from model~\eqref{eq:model} with a super-Gaussian density. For different values of the number of samples, we plot the Amari distance and the fitting time. Thick lines link median values across seeds. \textbf{Right: Robustness w.r.t intra-subject variability in MEG.}
     (\textbf{top}) $\ell_2$ distance between shared components corresponding to the same stimuli in different trials.  (\textbf{bottom}) Fitting time.}
\end{figure}
We generate components using a slightly super Gaussian density: $s_j = d(x)$ with $d(x) = x |x|^{0.2}$ and $x \sim \mathcal{N}(0, 1)$. We vary the number of samples $n$ between $10^2$ and $10^4$. We compute the mean Amari distance across subjects and record the computation time. The experiment is repeated $40$ times. We plot the Amari distance as a function of the computation time in Fig~\ref{exp:syn_timings}. Each point corresponds to the Amari distance/computation time for a given number of samples and a given seed. We then consider for a given number of samples, the median Amari distance and computation time across seeds and plot them in the form of a thick line.  From Fig~\ref{exp:syn_timings}, we see that ShICA-J is the method of choice when speed is a concern while ShICA-ML yields the best performance in terms of Amari distance at the cost of an increased computation time. The thick lines for ShICA-J and Multiset CCA are quasi-flat, indicating that the number of samples does not have a strong impact on the fitting time as these methods only work with covariances. On the other hand CanICA or MVICA computation time is more sensitive to the number of samples.

%


\paragraph{Robustness w.r.t intra-subject variability in MEG}

In the following experiments we consider the Cam-CAN
dataset~\cite{taylor2017cambridge}. We use the magnetometer data from the MEG of
$m=100$ subjects chosen randomly among 496.
In appendix~\ref{app:preprocessing}
we give more information about Cam-CAN dataset.
Each subject is repeatedly presented three audio-visual stimuli. 
For each stimulus, we divide the trials into two sets and within each set, 
the MEG signal is averaged across trials to isolate the evoked response. This
procedure yields 6 chunks of individual data (2 per stimulus).
%
%
We study the similarity between shared components corresponding to repetitions of the same stimulus. This gives a measure of robustness of each ICA algorithm with respect
to intra-subject variability.
Data are first reduced using a subject-specific PCA with $p=10$ components.
The initial dimensionality of the data before PCA is $102$ as we only use the 102 magnetometers.
Algorithms are run 10 times with different seeds on the 6 chunks of data,
and shared components are extracted.
When two chunks of data correspond to repetitions of the same stimulus they should yield similar
components.
For each component and for each stimulus, we therefore measure the $\ell_2$
distance between the two repetitions of the stimulus.
 This yields $300$ distances per algorithm that are
plotted on Fig~\ref{fig:eeg_intragroup_variability}.

The components recovered by ShICA-ML have a much lower variability than other approaches. The performance of ShICA-J is competitive with MVICA while being much faster to fit. Multiset CCA yields satisfying results compared with ShICA-J. However we see that the number of components that do not match at all across trials is greater in Multiset CCA.
Additional experiments on MEG data are available in Appendix~\ref{app:phantom}.

\paragraph{Reconstructing the BOLD signal of missing subjects}
\begin{wrapfigure}{l}{.42\textwidth}
  \centering
  \includegraphics[width=0.99\linewidth]{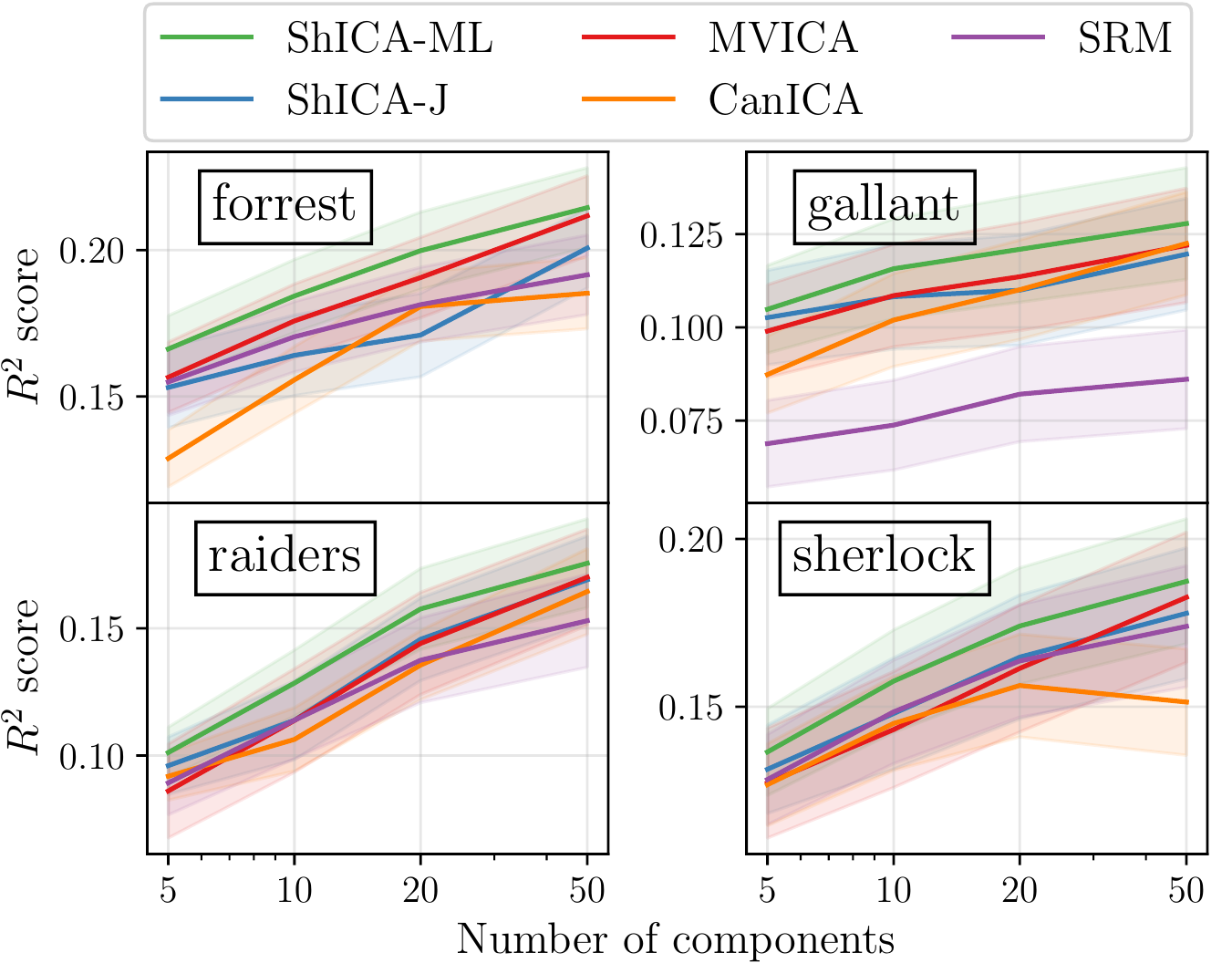}
  \includegraphics[width=0.99\linewidth]{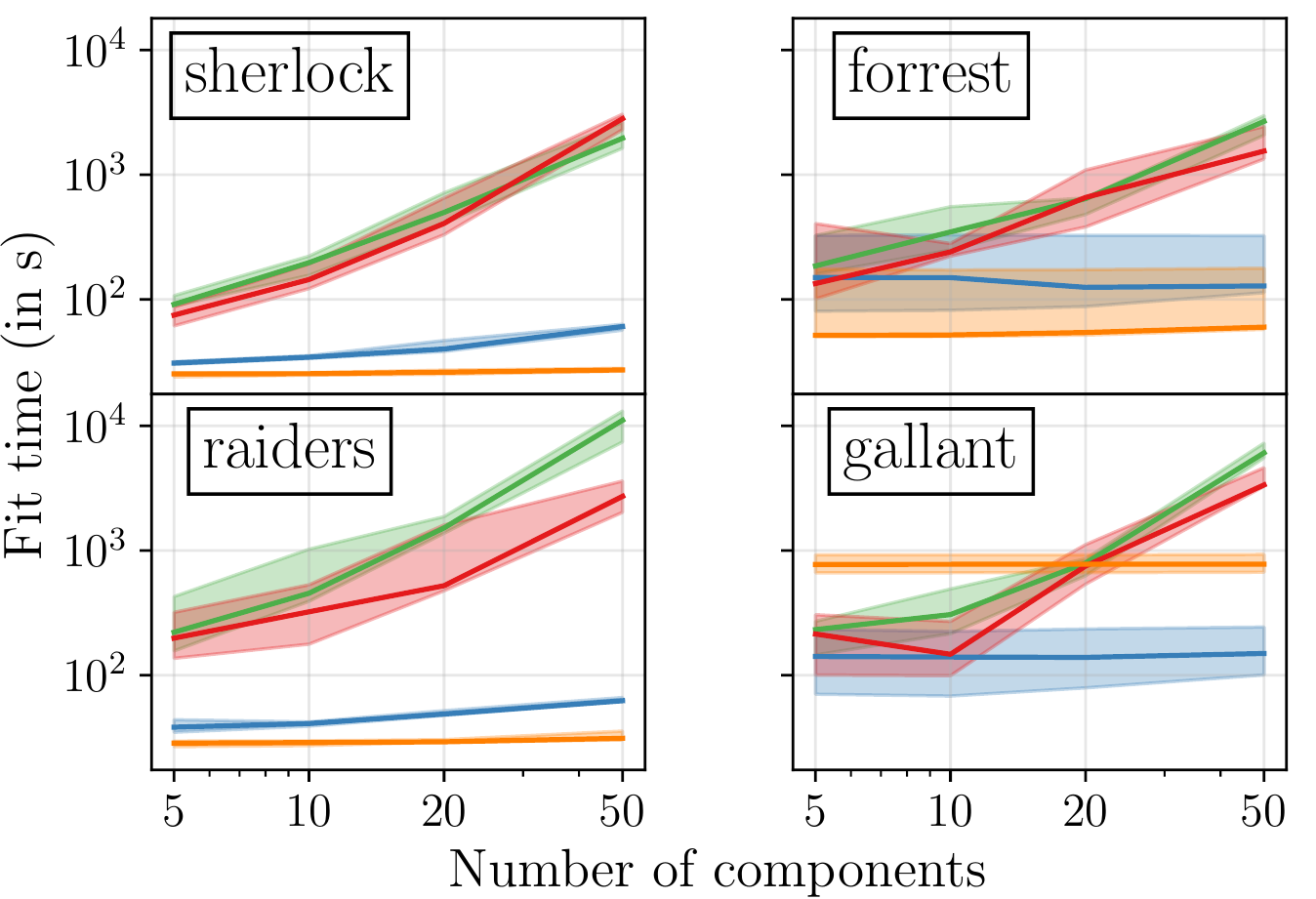}
  
  \caption{\textbf{Reconstructing the BOLD signal of
      missing subjects}. (\textbf{top}) Mean $R^2$ score between reconstructed data and true
    data. (\textbf{bottom}) Fitting time.
    }
  \label{fig:reconstruction}
\end{wrapfigure}
We reproduce the experimental pipeline of~\cite{richard2020modeling} to benchmark GroupICA methods using their ability to reconstruct fMRI data of a left-out subject.
The preprocessing involves a dimension reduction step performed using the shared response model~\cite{chen2015reduced}. Detailed preprocessing pipeline is described in Appendix~\ref{app:preprocessing}. We call an \emph{unmixing operator} the product of the dimension
reduction operator and an unmixing matrix and a \emph{mixing operator} its pseudoinverse. There is one unmixing operator and one mixing operator per view.
The unmixing operators are learned using all subjects
and $80\%$ of the runs. Then they are applied on the remaining $20\%$ of the runs using $80\%$
of the subjects yielding unmixed data from which shared components are
extracted.
The unmixed data are combined by averaging (for SRM and other baselines) or using the MMSE estimate for ShICA-J and ShICA-ML.
We
then apply the mixing operator of the remaining $20\%$ subjects on the shared components to reconstruct their data.
Reconstruction accuracy is measured via the coefficient of determination, \aka $R^2$ score, that
yields for each voxel the relative discrepancy between the true time course and the predicted one.
%
For each compared algorithm, the experiment is run 25 times with different seeds to obtain error bars. We report the mean $R^2$ score across voxels in a region of interest (see Appendix~\ref{app:preprocessing} for details)
 and display the results in Fig~\ref{fig:reconstruction}. The error bars represent a $95\%$ confidence interval.
The chance level is given by the $R^2$ score of an algorithm that samples the
coefficients of its unmixing matrices and dimension reduction operators from a
standardized Gaussian. The median chance level is below $10^{-3}$ on all
datasets.
ShICA-ML yields the best $R^2$ score in all datasets and for any number of
components. ShICA-J yields competitive results with respect to MVICA
while being much faster to fit. A popular benchmark especially in the SRM
community is the time-segment matching experiment~\cite{chen2015reduced}: we
include such experiments in Appendix~\ref{app:timesegment}.
In
appendix~\ref{app:table}, we give the performance of ShICA-ML, ShICA-J and MVICA
in form of a table.


%


\section{Conclusion, Future work and Societal impact}

We introduced the ShICA model as a principled unifying solution to the problems of shared response modelling and GroupICA. ShICA is able to use both the diversity of Gaussian variances and non-Gaussianity for optimal estimation. We presented two algorithms to fit the model: ShICA-J, a fast algorithm that uses noise diversity, and ShICA-ML, a maximum likelihood approach that can use non-Gaussianity on top of noise diversity. ShICA algorithms come with principled procedures for shared components estimation, as well as adaptation and estimation of noise levels in each view (subject) and component. On simulated data, ShICA clearly outperforms all competing methods in terms of the trade-off between statistical accuracy and computation time. On brain imaging data, ShICA gives more stable decompositions for comparable computation times, and more accurately predicts the data of one subject from the data of other subjects, making it a good candidate to perform transfer learning. 
Our code is available at \url{https://github.com/hugorichard/ShICA}.
\footnote{Regarding the ethical aspects of this work, we think this work presents exactly the same issues as any brain imaging analysis method related to ICA.}
\clearpage
\paragraph{Acknowledgement and funding disclosure}
This work has received funding
from the European Union’s Horizon 2020 Framework Programme for Research and Innovation under
the Specific Grant Agreement No. 945539 (Human Brain Project SGA3), the KARAIB AI chair
(ANR-20-CHIA-0025-01), the Grant SLAB ERC-StG-676943 and the BrAIN AI chair (ANR-20-CHIA-0016). PA acknowledges funding by the French government under management of Agence Nationale de la Recherche as part of the “Investissements d’avenir” program, reference ANR19-P3IA-0001 (PRAIRIE 3IA Institute). AH received funding from a CIFAR Fellowship. 
\bibliographystyle{plain}
\bibliography{biblio}


\newpage
\appendix

\section{Proofs and Lemmas}
\subsection{Proof of Theorem~\ref{thm:identif}}
\label{proof:identif}
\begin{proof}
By hypothesis, the covariances verify $C_{ij} =\bbE[\xb_i\xb_j^\top] = A_i(I_p + \delta_{ij}\Sigma_i)A_j^{\top} = A'_i(I_p + \delta_{ij}\Sigma'_i){A'_j}^{\top}$ for all $i, j$. We let $P_i=A_i^{-1}A_i'$. The previous relationship for $j\neq i$ gives $P_iP_j^{\top} = I_p$. Because there are more than 3 views, there is another integer $k \notin\{i,j\}$, and we have $P_iP_k^{\top}= P_jP_k^{\top}=I_p$. This shows that $P_i = P_j$: all these matrices are equal, and we call $P$ their common value. The previous equation also gives $PP^{\top} = I_p$, so $P$ is orthogonal. 
We have that $s + n_i$ and $s' + n_i'$ have independent components and $s + n_i = P(s' + n_i')$. Lemma~\ref{lemma:ica} (a direct consequence of classical ICA results~\cite{comon1994independent}, Theorem 10) gives $P=\Pi^{-1} \Omega \Pi'$ where $\Pi$ and $\Pi'$ are sign and permutation matrices such that the first $g$ components of $\Pi(s + n_i)$ and $\Pi'(s' + n_i')$ are Gaussian, and $\Omega$ is a block diagonal matrix given by
\[
\Omega = \begin{bmatrix} \Omega_g & 0 \\ 0 & I_{p - g} \end{bmatrix}
\]
where $\Omega_g$ is orthogonal.
We call $A^{(g)}$ the first $g \times g$ block of a matrix $A$ so that $\Omega^{(g)} = \Omega_g$.

Then, considering only the Gaussian components, we can write for $i=j$:  
$(\Pi \Sigma_i)^{(g)} = \Omega_g (\Pi' \Sigma'_i)^{(g)} \Omega_g^{\top}$ for all $i$. This, combined with Assumption~\ref{ass:diversity}, implies that $\Omega_g$ is a sign and permutation matrix (see Lemma~\ref{lemma:eigdecomp}) and therefore $P$ is a sign and permutation matrix. Then it follows that $I + \Sigma_i = P(I + \Sigma'_i)P^{\top}$ and therefore $\Sigma_i = P \Sigma'_i P^{\top}$ so $\Sigma'_i = P^{\top} \Sigma_i P$.
\end{proof}

\subsection{Proof of Theorem~\ref{th:eig}}
\label{proof:eig}
\begin{proof}
  Let us denote $W \in \bbR^{mp \times mp}$ the block diagonal matrix with block $i$ given by
  $(A^i)^{-1}$. We have $\mathcal{C} \ub = \lambda \mathcal{D} \ub  \iff W \mathcal{C} W^\top \zb = \lambda W \mathcal{D} W^\top \zb
  $ where $\ub = W^\top \zb$. We call $\zb$ a reduced eigenvector.
  Each
  block in $W \mathcal{C} W^\top$ and in $W \mathcal{D} W^\top$ is diagonal so any reduced eigenvector $\zb = \begin{bmatrix} \zb_1 \\ \vdots \\ \zb_m \end{bmatrix}$ is
  such that the matrix $Z = [\zb_1 \dots \zb_m]$ has exactly one non-zero line.
  Following Lemma~\ref{lemma:nonzerocoord}, the first $p$ leading reduced
  eigenvectors $\zb^1, \dots, \zb^p$ all have different first non-zero coordinates.
  Therefore the concatenation of the first $p$ leading reduced eigenvectors is given
  by $[\zb^1, \dots \zb^p] = \begin{bmatrix} \Gamma_1 \\ \vdots \\ \Gamma_m \end{bmatrix} P^{\top}$ where $P^{\top} \in \mathbb{R}^{p \times p}$ is a permutation matrix and $\Gamma_i
  \in \mathbb{R}^{p \times p}$ is a diagonal matrix. Therefore, the first $p$
  eigenvectors are given by $[\ub^1 \dots \ub^p] = \begin{bmatrix} W_1^{\top} \\ \vdots \\ W_m^{\top} \end{bmatrix} = \begin{bmatrix} (A_1^{-1})^{\top} \Gamma_1 P^{\top} \\ \vdots \\ (A_m^{-1})^{\top} \Gamma_m P^{\top} \end{bmatrix}$  and so $W_i = P \Gamma_i A_i^{-1}$
\end{proof}

\begin{lemma}
\label{lemma:ica}
Let $\sbb \in \mathbb{R}^k$ and $\sbb'\in \mathbb{R}^k$ have independent components among which $g$ are Gaussian, and $P$ a rotation matrix such that $\sbb = P\sbb'$. Then, $P=\Pi^{-1} O \Pi'$ where $\Pi$ and $\Pi'$ are sign and permutation matrices such that the first $g$ components of $\Pi \sbb$ and $\Pi' \sbb'$ are Gaussian and $O$ is a block diagonal matrix such that $O^{(g)}$, the first $g \times g$ block of $O$, is orthogonal and the other block is identity.
\end{lemma}
\begin{proof}
  From~\cite{comon1994independent}, Theorem 10:
  Assume $\sbb = P\sbb'$, if the column $j$ of $P$ has more than one non-zero element then $s'_j$ is Gaussian. 
  
  Let us define permutations $\Pi_1$, $\Pi'_1$ such that the first $g$ components of $\Pi_1 \sbb$ and $\Pi'_1 \sbb'$ are Gaussian and $P_1  = \Pi_1 P (\Pi'_1)^{-1}$. We can see that $P_1$ is orthogonal.
  
  We have $\Pi_1 \sbb = P_1 \Pi'_1 \sbb'$. So the last $p-g$ columns of $P_1$ contain at most one non-zero element. Using orthogonality of $P_1$ this non-zero element has value $1$ or $-1$ and is also the only one in its line. Let us focus on column $l > g$. Assume column $l$ has its non-zero element at index $k \leq g$. Then line $k$ in $P_1$ is only non-zero at index $l$ and therefore $(\Pi_1 \sbb)_k$ (which is Gaussian) is equal to $(\Pi'_1 \sbb')_l$ (which is not). Therefore column $l$ can only have its non-zero element at an index greater than $g$. This shows that $P_1$ is block diagonal $P_1 = \begin{bmatrix} O_g & 0 \\ 0 & P_2 \end{bmatrix}$ where $O_g$ is orthogonal  and $P_2$ is a sign and permutation matrix.
  \begin{align}
      &\begin{bmatrix} O_g & 0 \\ 0 & P_2 \end{bmatrix} = \Pi_1 P (\Pi'_1)^{-1} \\
      & \iff \begin{bmatrix} O_g & 0 \\ 0 & I \end{bmatrix} \begin{bmatrix} I & 0 \\ 0 & P_2 \end{bmatrix}  = \Pi_1 P (\Pi'_1)^{-1} \\
      & \iff \Pi_1^{-1} \begin{bmatrix} O_g & 0 \\ 0 & I \end{bmatrix} \begin{bmatrix} I & 0 \\ 0 & P_2 \end{bmatrix} \Pi'_1  = P
  \end{align}
  
  Therefore setting $\Pi' =   \begin{bmatrix} I & 0 \\ 0 & P_2 \end{bmatrix} \Pi'_1$ and $\Pi = \Pi_1$ and $O= \begin{bmatrix}O_g & 0 \\ 0 & I  \end{bmatrix}$ concludes the proof.
  
\end{proof}

\begin{lemma}
\label{lemma:eigdecomp}
Assume that Assumption 2 holds for $\Sigma_i$, and that there is an orthogonal matrix $P$ and diagonal matrices $\Sigma_i'$ such that for all $i$, $\Sigma_i' = P\Sigma_iP^{\top}$. Then, $P$ is a permutation matrix.
\end{lemma}
\begin{proof}
The proof is in two parts. First, we show that there exist some coefficients $\alpha_1, \dots, \alpha_m$ such that the matrix $\sum_i\alpha_i\Sigma_i$ has distinct coefficients on the diagonal. Then, since we have $\sum_i\alpha_i\Sigma'_i = P\left(\sum_i\alpha_i\Sigma_i\right)P^{\top}$, and the diagonal $\sum_i\alpha_i\Sigma_i$ has distinct entries, we can invoke the unicity of the eigenvalue decomposition for symmetric matrices, which shows that $P$ is necessarily a permutation matrix.
Now, the only thing left is to prove is that Assumption 2 implies the existence of this linear combination.

We assume by contradiction that any linear combination of the $\Sigma_i$ has two equal entries.

For $\alpha = [\alpha_1, \dots, \alpha_m]$, we let $\mathcal{S}(\alpha) = \diag(\sum_i\alpha_i\Sigma_i)\in\bbR^p$, where $\diag(\cdot)$ extracts the diagonal entries. The operator $\mathcal{S}$ is linear.
We now define for $j, j'\leq p$ the linear form $\ell_{jj'}(\alpha) = \mathcal{S}(\alpha)_j - \mathcal{S}(\alpha)_{j'}\in\bbR$. The assumption on the linear combinations of $\Sigma_i$ simply rewrites:
For all $\alpha\in\bbR^m$, there exists $j, j'\leq p$ such that $\ell_{jj'}(\alpha) = 0$.

From a set point of view, this relationship writes
$$
\bigcup_{j, j'}\mathrm{Ker}(\ell_{jj'}) = \bbR^m\enspace.
$$
Since the $\ell_{jj'}$ are all linear forms, the $\mathrm{Ker}(\ell_{jj'})$ are subspaces of dimensions $m$ or $m-1$, and since their union is of dimension $m$, there exists $j, j'$ such that  $\mathrm{Ker}(\ell_{jj'}) = \bbR^m$, i.e. such that $\ell_{jj'} = 0$.

As a consequence, we have for all $\alpha$, $\mathcal{S}(\alpha)_j = \mathcal{S}(\alpha)_{j'}$. This implies that the sequences $(\Sigma_{ij})_i$ and $(\Sigma_{ij'})_i$ are equal, which contradicts Assumption 2.

We have therefore shown that Assumption 2 implies the existence of a linear combination of the $\Sigma_i$ that has distinct entries, which concludes the proof.
\end{proof}

\begin{lemma}
\label{lemma:nonzerocoord}
Let us consider the following eigenvalue problem:
 \begin{align}
  & \begin{bmatrix} I + \Sigma_1 & I & \dots & I \\
    I & I + \Sigma_2 & \ddots & \vdots \\
    \vdots &  \ddots & \ddots & I  \\
    I & \dots & I  &I + \Sigma_m
  \end{bmatrix} \zb = \lambda \begin{bmatrix}
    I + \Sigma_1 & 0 & \dots  & 0 \\
    0 & I + \Sigma_2 & \ddots & \vdots \\
    \vdots & \ddots & \ddots & 0 \\
    0& \dots  & 0 &  I + \Sigma_m  \\
  \end{bmatrix} \zb
  \label{reducedeig}
\end{align}
where $\forall i, \enspace 1 \leq i \leq m, \enspace  \Sigma_m \in \bbR^{p, p}$ are positive diagonal matrices and I is the identity matrix.
If the first $p$ eigenvalues are distincts, the first $p$ eigenvectors $\zb^1, \dots, \zb^p, \zb^i \in \mathbb{R}^{mp}$ have different first non-zero coordinates.
\end{lemma}
\begin{proof}
We sort the eigenvectors in $p$ groups of $m$ vectors so that all
vectors in group $l$ have their $l$-th coordinate
different from 0.
Let $\zb^{(l)}$ be an eigenvector in group $l$ and let us call $\wb_l \in
\mathbb{R}^{m}$ the non-zero coordinates of this eigenvector: $\forall i \in \{1 \dots m \}, w_{li} = z^{(l)}_{l + (i-1)p}$.

We have:
\begin{align}
\begin{bmatrix}
  1 + \Sigma_{1l} & 1 & \dots & 1  \\
  1 & 1 + \Sigma_{2l} & \ddots  &\vdots \\
  \vdots & \ddots & \ddots & 1  \\
  1 & \dots & 1 & 1 + \Sigma_{ml}  \\
\end{bmatrix} \wb_l =  \begin{bmatrix}
  1 + \Sigma_{1l} & 0 & \dots  & 0 \\
  0 & 1 + \Sigma_{2l} & \ddots & \vdots \\
  \vdots & \ddots & \ddots & 0 \\
  0& \dots  & 0 &  1 + \Sigma_{ml}  \\
\end{bmatrix} \wb_l \lambda_l
\label{eigsimp}
\end{align}

We now show that the biggest eigenvalue of~\eqref{eigsimp} is strictly above 1 while all
others are strictly below 1. The core of the proof comes from the study of the eigenvalues of a matrix modified by a rank 1 matrix. The reasoning we use here follows~\cite{golub1973some} (end of section 5).

Let us introduce 
$K^l = \diag(\Sigma_{1l} \dots \Sigma_{ml})$ and $\ub = \begin{bmatrix} 1 \\ \vdots \\ 1 \end{bmatrix}$.
Let us drop the index $l$ in the notations for simplicity.

The problem can be rewritten
\begin{align}
  &(\ub \ub^{\top} + K) \wb =  (I + K) \wb \lambda \\
  & \iff (I + K)^{-1}(\ub \ub^{\top} + K) \wb =   \wb \lambda
\end{align}

The characteristic polynomial is given by:
\begin{align}
  &\mathcal{P}(\lambda) = \det( (I + K)^{-1} K - \lambda I + (I + K)^{-1} \ub \ub^{\top}) \label{carpol} \\
  &\propto \det( I + ((I + K)^{-1} K - \lambda I)^{-1}(I + K)^{-1} \ub \ub^{\top})
\end{align}
where we implicitly focus here on eigenvalues $\lambda$ such that $\det((I + K)^{-1} K - \lambda I) \neq 0 \iff \forall i, \lambda \neq \frac{k_i}{1 + k_i}$.

We then use the following property:
Let $A \in \mathbb{R}^{a, b}$ and $B \in \mathbb{R}^{b, a}$ we have
$\det(I_a + AB) = \det(I_b + BA)$.

Let us call $\chi(\lambda) = \det( I + ((I + K)^{-1} K - \lambda I)^{-1}(I + K)^{-1} \ub \ub^{\top})$ we have:
\begin{align}
\chi(\lambda)
  &= 1 + \ub^{\top}((I + K)^{-1} K - \lambda I)^{-1}(I + K)^{-1} \ub \\
  &= 1  + \sum_{i=1}^m \frac1{1 + k_i} \frac1{ \frac{k_i}{1 + k_i} - \lambda}
\end{align}
where $k_i = \Sigma_{il} > 0$.
Taking the derivative we get 
\begin{align}
\chi'(\lambda) = \sum_{i=1}^m \frac1{1 + k_i} \frac1{ (\frac{k_i}{1 + k_i} - \lambda)^2} > 0
\end{align}

Trivially, $\forall i, \frac{k_i}{1 + k_i} < 1$. We also have
\begin{align}
  \chi(1) = 1 + \sum_{i=1}^{m} \frac1{1 + k_i} \frac1{ \frac{k_i}{1 + k_i} - 1} = 1 - m < 0
\end{align}
 and $\lim_{\lambda \rightarrow + \infty} \chi(\lambda) = 1$ so as $\chi$
 is continuous and strictly increasing on $[1, +\infty[$. Therefore, it reaches $0$ only once on this interval (excluding 1 since we know $\chi(1) \neq 0$). Therefore the greatest eigenvalue $\lambda^*$ is strictly above $1$ while all other eigenvalues are strictly below $1$.
 
  Note that because $\chi' > 0$, $\lambda^*$ is of multiplicity $1$. In the analysis above we ignored those eigenvalues $\lambda$ such that $\lambda = \frac{k_i}{1 + k_i}$ for some $i$. However since $\frac{k_i}{1 + k_i} < 1$, none of these eigenvalues can be the largest one.
 
 Finally, the $p$ first eigenvectors belong to different groups (the
 corresponding eigenvalues are all strictly above 1). This shows that these eigenvectors have
 different first non-zero coordinates. 
 
\end{proof}

\section{Identifiability results for $m< 3$}
\label{app:identifiability}
We have a slightly weaker identifiability result when $m=2$.
\begin{proposition}
  \label{prop:identifiability_2d}
  Let $m=2$, and suppose that the scalars $(1 + \Sigma_{1j})(1+\Sigma_{2j})$ for $j=1\dots p$ are all different. We let $\Theta'=(A_1', A_2', \Sigma_1',\Sigma_2')$ that also generates $\xb_1, \xb_2$. Then, there exists a permutation and scale matrix $P$ such that $A'_1 =A_1P$ and $A'_2 = A_2P^{-\top}$.
\end{proposition}
\begin{proof}
  We let $P=A_1^{-1}A_1'$. Since $C_{12} = I_p$, it holds 
  $A_2^{-1}A_2'= P^{-\top}$. Then, we have
  $I_p + \Sigma_1 = P(I_p + \Sigma'_1)P^{\top}$. This means that there exists $U\in\mathcal{O}_p$ such that $P = (I_p + \Sigma_1)^{\frac12}U(I_p + \Sigma'_1)^{-\frac12}$. Since $P^{-\top}(I_p+\Sigma'_2)P^{-1} = I_p+\Sigma_2$, we find
  $U(I_p+\Sigma'_1)(I_p+\Sigma'_2)U^{\top} = (I_p+\Sigma_1)(I_p+\Sigma_2)$. By identification, $U$ is a permutation matrix, and $P$ is a scale and permutation matrix.
\end{proof}
As a consequence, when there are only two subjects, it is possible to recover the components and noise levels up to a scaling factor.
When there is only one view, $m=1$, there is a global rotation indeterminacy: 
$
A_1(I_p + \Sigma_1)A_1^{\top} = A'_1(I_p + \Sigma_1){A'_1}^{\top}
$
for $A'_1 = A_1(I_p + \Sigma_1)^{\frac12}U(I_p + \Sigma_1)^{-\frac12}$ where $U$ is any orthogonal matrix. In this case, we lose identifiability.

\section{Derivation of fixed point updates for scalings}

We want to minimize

\begin{align}
\label{app:fixedpoint}
L((\Phi_i)_{i=1}^m)= \sum_i \sum_{j \neq i} \| \Phi_i \diag(Y_{ij}) \Phi_j - I_p \|_F^2
\end{align}
for $\Phi_i$ diagonal. With respect to each $\Phi_i$, this function is strongly-convex, which means that the minimization w.r.t $\Phi_i$ can be done by cancelling the gradient.
The gradient is given by
\begin{align}
\partialfrac{\Phi_i}{L} = 2\sum_{j \neq i} (\Phi_i \diag(Y_{ij}) \Phi_j - I_p) \Phi_j
\end{align}
Therefore we get
\begin{align}
&\partialfrac{\Phi_i}{L} = 0 \\
&\iff 2\sum_{j \neq i} (\Phi_i \diag(Y_{ij}) \Phi_j - I_p) \Phi_j = 0 \\
&\iff \Phi_i \sum_{j \neq i} \diag(Y_{ij}) \Phi_j^2 - \sum_{j \neq i} \Phi_j = 0 \\
&\iff \Phi_i = \frac{\sum_{j \neq i} \Phi_j}{\sum_{j \neq i} \diag(Y_{ij}) \Phi_j^2}
\end{align}

This update equation ensures that $\Phi_i = \arg\min_{\Phi_i} L((\Phi_j)_{i=1}^m)$, and we then loop through the $\Phi_i$ to get an alternate minimization scheme, which is guaranteed to converge to a stationary point of~\eqref{app:fixedpoint}.


\section{EM E-step and M-step for ShICA with Gaussian components}
\subsection{E-step}
\label{conditional_density}
The derivations are the same as in section~\ref{app:emestep1} but the sum over
$\alpha \in {\frac12, \frac32}$ is replaced by just $\alpha=1$.

\subsection{M-step}
The function to minimize in the M-step is then given by:
\begin{align}
  \Jcal &= -\log p(\xb, \sbb) \\
  &= \sum_{i=1}^m \log(|\Sigma_i|) + \frac12 \tr(\Sigma_i^{-1} \left[(\yb_i - \EE[\sbb | \xb]) (\yb_i - \EE[\sbb | \xb])^{\top}+ \VV[\sbb | \xb]\right]) + c
\end{align}
where $c$ does not depend on $\Sigma_i$

Therefore we get closed-form updates for $\Sigma_i$: 
\begin{align}
\Sigma_i \leftarrow  \diag((\yb_i - \EE[\sbb | \xb]) (\yb_i - \EE[\sbb | \xb])^{\top}+ \VV[\sbb | \xb])
\end{align}
Plugging in the closed-form formula for $\EE[\sbb|\xb]$ and $\VV[\sbb|\xb]$ we get updates that only depends on the covariances $\hat{C_{ij}} = \EE[\xb_i \xb_j^{\top}]$.
\begin{align*}
\Sigma_i \leftarrow \diag(\hat{C_{ii}} - 2 \VV[\sbb | \xb]  \sum_{j=1}^m \Sigma_j^{-1} \hat{C}_{ji}  + \VV[\sbb | \xb]  \sum_{j = 1}^m \sum_{l = 1}^m \left(\Sigma_j^{-1} \hat{C}_{jl} \Sigma_l^{-1} \right) \VV[\sbb | \xb] + \VV[\sbb | \xb])
\end{align*}

\section{EM E-step and M-step for ShICA with non-Gaussian components}
\label{app:emestep}

  \subsection{E-step}
  \label{app:emestep1}
  The complete likelihood is given by
\begin{align}
  p(\xb, \sbb) &= \prod_i p(\xb_i | \sbb) p(\sbb) \\
  &= \prod_i p(\xb_i | \sbb) \prod_j \sum_{\alpha \in \{0.5, 1.5\}}p(s_j | \alpha) \\
\end{align}
where 
\begin{align}
  p(s_j | \alpha) = \Ncal(s_j; 0, \alpha)
\end{align}

We have
\begin{align}
  p(\xb_i | \sbb) &=|W_i|\Ncal(\yb_i; \sbb, \Sigma_i)  \\
                  &= |W_i| \prod_j \Ncal(y_{ij}; s_j, \Sigma_{ij})
\end{align}
where $\Sigma_{ij}$ is the coefficient $j, j$ of $\Sigma_i$ and $\yb_i = W \xb_i$.

Let us introduce a first lemma:
\begin{lemma}
  \label{lemma:multmgauss}
  \begin{align*}
    \prod_{i=1}^m \Ncal(x_i; u, v_i) = \prod_{i=1}^m \Ncal(x_i; \bar{x}, v_i) \sqrt{2 \pi \bar{v}} \Ncal(\bar{x}; u, \bar{v})
  \end{align*}
  where $\bar{v} = (\sum_{i=1}^m v_i^{-1})^{-1}$ and $\bar{x} = \frac{\sum_i v_i^{-1}
    x_i}{\sum_i v_i^{-1}}$.
\end{lemma}
\begin{proof}
  We have that
  \begin{align}
    \sum_i \frac1{v_i}(x_i - u)^2 &= \sum_i\frac1{v_i}(x_i - u)^2 \\
                                  &= \sum_i \frac1{v_i}(x_i - \bar{x} + \bar{x} - u)^2 \\
                                  &= \sum_i \frac1{v_i}(x_i - \bar{x})^2 + \sum_i \frac1{v_i}(\bar{x} - u)^2 \label{eq:lem:multigauss}
  \end{align}
  and therefore
  \begin{align}
    &\prod_i (\frac1{\sqrt{2 \pi v_i}}\exp(-\frac1{2v_i}(x_i - \mu)^2)) \\
    &= \prod_i \frac1{\sqrt{2 \pi v_i}}\exp(\sum_i -\frac12 (\frac1{v_i}(x_i - \bar{x})^2 + \frac1{v_i}(\bar{x} - u)^2)) \\ 
    &= \prod_i \Ncal(x_i, \bar{x}, v_i) \exp(-\frac12( \sum_i \frac1{v_i})(\bar{x} - u)^2)) \\ 
  \end{align}
  so the desired result follow.
\end{proof}

By Lemma~\ref{lemma:multmgauss}, we have
\begin{align}
  \prod_i p(\xb_i | \sbb) &= \prod_i |W_i| \prod_j \Ncal(y_{ij}; \bar{y}_{j}, \Sigma_{ij}) \sqrt{2 \pi \bar{\Sigma_{j}}} \Ncal(\bar{y}_j; s_j, \bar{\Sigma_{j}})  \\
\end{align}
where $\bar{y}_j = \frac{\sum_i \Sigma_{ij}^{-1} y_{ij}}{ \sum_i
  \Sigma_{ij}^{-1}}$ and $\bar{\Sigma_{j}} = (\sum_i
\Sigma_{ij}^{-1})^{-1}$.
Hiding variable that do not depend on $\sbb$ we obtain

\begin{align}
  \prod_i p(\xb_i | \sbb) &\propto \prod_j \Ncal(\bar{y}_j; s_j, \bar{\Sigma_{j}})  \\
\end{align}

Then we get
\begin{align}
  p(\xb, \sbb) &\propto \prod_j \sum_{\alpha \in \{0.5, 1.5\}} \Ncal(s_j; \bar{y}_j, \bar{\Sigma_{j}}) \Ncal(s_j; 0, \alpha)
\end{align}

Let us now prove a second Lemma:
\begin{lemma}
  \label{lemma:multgauss}
  \begin{align*}
    \mathcal{N}(x; y, \nu) \mathcal{N}(x, 0, \alpha) = \mathcal{N}(y; 0, \nu + \alpha) \mathcal{N}(x; \frac{\alpha y}{\alpha + \nu}, \frac{\nu \alpha}{\alpha + \nu})
  \end{align*}
\end{lemma}
\begin{proof}

  We have 
\begin{align}
  &\mathcal{N}(x; y, \nu) \mathcal{N}(x, 0, \alpha) = \frac{\exp \left(-\frac{(x-y)^2}{2\nu} \right)}{\sqrt{2 \pi \nu}} \frac{\exp \left(-\frac{x^2}{2\alpha}\right)}{\sqrt{2 \pi \alpha}}
\end{align}

Then,
\begin{align}
  &\exp \left(-\frac{(x-y)^2}{2\nu} \right) \\
  &= \exp\left(- \frac{\alpha (x-y)^2 + \nu x^2}{2 \alpha \nu} \right) \\
  &= \exp\left(- \frac{\alpha (x^2 -2xy + y^2) + \nu x^2}{2 \alpha \nu} \right) \\ 
  &=\exp\left(- \frac{x^2(\alpha + \nu) -2x(\alpha y) + \alpha y^2 }{2 \alpha \nu} \right) \\
  &= \exp\left( -\frac{x^2 -2x\frac{\alpha y}{\alpha + \nu} + \frac{\alpha y^2}{\alpha + \nu} }{2 \frac{\alpha \nu}{\alpha + \nu}} \right) \\
  &= \exp\left( -\frac{(x - \frac{\alpha y}{\alpha + \nu})^2 - ( \frac{\alpha y}{\alpha + \nu} )^2 + \frac{\alpha y^2}{\alpha + \nu} }{2 \frac{\alpha \nu}{\alpha + \nu}} \right) \\
  &= \exp\left( -\frac{(x - \frac{\alpha y}{\alpha + \nu})^2}{2\frac{\alpha \nu}{\alpha + \nu}}\right) \exp \left(-\frac{ - \alpha^2 y^2 + (\alpha + \nu)\alpha y^2 }{2 \alpha \nu(\alpha + \nu)} \right) \\
  &= \exp\left( -\frac{(x - \frac{\alpha y}{\alpha + \nu})^2}{2\frac{\alpha \nu}{\alpha + \nu}}\right) \exp \left(-\frac{\nu\alpha y^2 }{2 \alpha \nu(\alpha + \nu)} \right)
\end{align}

and
\begin{align}
  \frac1{\sqrt{2 \pi \nu}\sqrt{2 \pi \alpha}} &= \frac1{\sqrt{2 \pi (\nu + \alpha)}\sqrt{2 \pi \frac{\nu \alpha}{\nu + \alpha}}}
\end{align}
so that the desired result follow.
\end{proof}
                 
By Lemma~\ref{lemma:multgauss}, we have:
\begin{align}
  &p(\xb, \sbb) \\
  &\propto \prod_j \sum_{\alpha \in \{0.5, 1.5\}} \Ncal(\bar{y}_{j}; 0 , \bar{\Sigma}_{j} + \alpha) \Ncal(s_j; \frac{\alpha \bar{y}_{j}}{\alpha + \bar{\Sigma_{j}}}, \frac{\bar{\Sigma_{j}}\alpha}{\alpha + \bar{\Sigma_{j}}})
\end{align}

and therefore we get:
\begin{align}
  p(\sbb | \xb) &= \frac{p(\sbb, \xb)}{\int_{\sbb} p(\sbb, \xb)} \\
                &= \prod_j \frac{\sum_{\alpha \in \{0.5, 1.5\}} \theta_{\alpha} \Ncal(s_j; \frac{\alpha \bar{y}_{j}}{\alpha + \bar{\Sigma_{j}}}, \frac{\bar{\Sigma_{j}}\alpha}{\alpha + \bar{\Sigma_{j}}})}{\sum_{\alpha \in \{0.5, 1.5\}} \theta_{\alpha}}
\end{align}
where $\theta_{\alpha} = \Ncal(\bar{y}_{j}; 0 , \bar{\Sigma}_{j} + \alpha)$.

So we obtain the desired result:
\begin{align}
  &\EE[s_j | \xb] = \frac{\sum_{\alpha \in \{0.5, 1.5\}} \theta_{\alpha} \frac{\alpha \bar{y}_{j}}{\alpha + \bar{\Sigma_{j}}}}{\sum_{\alpha \in \{0.5, 1.5\}} \theta_{\alpha}} \\
  & \VV[s_j | \xb] = \frac{\sum_{\alpha \in \{0.5, 1.5\}} \theta_{\alpha} \frac{\bar{\Sigma_{j}}\alpha}{\alpha + \bar{\Sigma_{j}}}}{\sum_{\alpha \in \{0.5, 1.5\}} \theta_{\alpha}}  
\end{align}

\subsection{M-step}
The function to minimize in the M-step is then given by:
\begin{align}
  \Jcal &= -\log p(\xb, \sbb) \\
  &= \sum_{i=1}^m - \log(|W_i|) + \log(|\Sigma_i|) + \frac12 \tr(\Sigma_i^{-1} \left[(\yb_i - \EE[\sbb | \xb]) (\yb_i - \EE[\sbb | \xb])^{\top}+ \VV[\sbb | \xb]\right]) + c
\end{align}
where $c$ does not depend on $\Sigma_i$ or $W_i$

Therefore we get closed-form updates for $\Sigma_i$: 
\begin{align}
\Sigma_i \leftarrow  \diag((\yb_i - \EE[\sbb | \xb]) (\yb_i - \EE[\sbb | \xb])^{\top}+ \VV[\sbb | \xb])
\end{align}

We update $W_i$ by performing a quasi-Newton step. 

We use the relative gradient $\mathcal{G}^{W_i}$ and $\mathcal{H}^{W_i}$ defined
by \\
$\mathcal{J}(W_i + \varepsilon W_i) = \mathcal{J}(W_i) + \langle
 \varepsilon|\mathcal{G}^{W_i}\rangle + \frac12 \langle
 \varepsilon|\mathcal{H}^{W_i} \varepsilon \rangle$.

 We get:
 \begin{align}
   \mathcal{J}(W_i + \varepsilon W_i) &= \sum_{i=1}^m \left[ -\log(|W_i|) -\log(|I_k + \varepsilon|) - \log(\mathcal{N}(\yb_i + \varepsilon \yb^i; \sbb; \Sigma_i)) \right] + const \\
                                      &= \mathcal{J}(W_i) - \tr(\varepsilon) + \frac12 \tr(\varepsilon^2) \\& \enspace \enspace + \frac12 \left[\langle \varepsilon \yb^i| (\Sigma_i)^{-1} (\yb^i - \sbb) \rangle + \langle (\yb^i - \sbb)| (\Sigma_i)^{-1} \varepsilon \yb^i \rangle + \langle \varepsilon \yb^i| (\Sigma_i)^{-1} \varepsilon \yb^i \rangle \right] \\ & \enspace + o(\|\varepsilon\|^2) \\
   &= \mathcal{J}(W_i) - \sum_a \varepsilon_{a, a} + \frac12 \sum_{a, b} \varepsilon_{a,b} \varepsilon_{b, a} \\& \enspace \enspace + \sum_{a, b} \varepsilon_{a, b} \left[(\Sigma_i)^{-1}(\yb^i - \sbb) (\yb^i)^{\top} \right]_{a, b} + \frac12 \sum_{a, b} \varepsilon_{a, b} \left[(\Sigma_i)^{-1} \varepsilon \yb^i (\yb^i)^{\top}\right]_{a, b} \\ & \enspace + o(\|\varepsilon\|^2) \\
   &= \mathcal{J}(W_i) - \sum_a \varepsilon_{a, a} + \frac12 \sum_{a, b} \varepsilon_{a,b} \varepsilon_{b, a} \\& \enspace \enspace + \sum_{a, b} \varepsilon_{a, b} \left[(\Sigma_i)^{-1}(\yb^i - \sbb) (\yb^i)^{\top} \right]_{a, b} + \frac12 \sum_{a, b, d} \varepsilon_{a, b} (\Sigma_i)^{-1}_{a, a} \varepsilon_{a, d} \left[\yb^i (\yb^i)^{\top}\right]_{d, b} \\ & \enspace + o(\|\varepsilon\|^2) \\
 \end{align}

 So:
 \begin{equation}
 \mathcal{G}^{W_i}_{a, b} =  -\delta_{a,b} + \left[(\Sigma_i)^{-1} (\yb^i - \sbb)(\yb^i)^{\top}\right]_{a, b}
 \end{equation}
 and
 \begin{equation}
 \mathcal{H}^{W_i}_{a, b, c, d} = \delta_{a, d}\delta_{b, c} + \delta_{a, c} \frac{y_{ib} y_{id}}{\Sigma_{ia}}
 \end{equation}
 
 We approximate the Hessian by
 \begin{align}
 \widehat{\mathcal{H}^{W_i}_{a, b, c, d}} = \delta_{ad} \delta_{bc} + \delta_{ac} \delta_{bd}\frac{(y_{ib})^2}{\Sigma_{ia}}
\end{align}
where the Hessian approximation is exact when the unmixed data have truly independent components.

Updates for $W_i$ are then given by
$W_i \leftarrow (I - \rho (\widehat{\mathcal{H}^{W_i}})^{-1} \mathcal{G}^{W_i}) W_i$, 
where $\rho$ is chosen by backtracking line-search.
We alternate between computing the statistics $\mathbb{E}[\sbb|\xb]$ and
$\Var[\sbb|\xb]$ (E-step) and updates of parameters $\Sigma_i$ and $W_i$ for $i=1 \dots m$ (M-step).

\section{Description of the datasets and the preprocessing pipeline}
\label{app:preprocessing}
All datasets are resampled and  masked using the brain mask available at
\url{http://cogspaces.github.io/assets/data/hcp_mask.nii.gz}. The dimensionality
of the data is given by the number of voxels in the mask: 212445. Data are
detrended and standardized so that each voxels' timecourse has zero mean and unit variance.  

When reconstructing the BOLD signal of missing subjects, data are preprocessed with a 6~mm smoothing.
In the timesegment matching experiment, we use unsmoothed data except for the sherlock dataset for which the available data are already smoothed.
Multiple acquisitions (called runs) are necessary to build the datasets. Each run lasts approximately 10 minutes.

Sherlock data are available at \url{http://arks.princeton.edu/ark:/88435/dsp01nz8062179}. We refer the reader to~\cite{chen2017shared} for a precise description of the study cohort,
experimental design and pre-processing pipeline.
The data are split manually into 4 runs of $395$ timeframes and one run of $396$
timeframes so that cross validation can be performed. Subject 5 is removed
because of missing data. The repetition time (TR) is 1.5s and the spatial
resolution is of 3~mm.

Forrest data are downloaded from OpenfMRI~\cite{poldrack2013toward}. Data
are acquired with a 7T scanner with an isotropic spatial resolution of 1~mm and
then resampled to a spatial resolution of 3~mm. A complete description of the
experimental design and study cohort are given in~\url{http://studyforrest.org}
and~\cite{hanke2014high}. Subject 10 is discarded as not all runs are available
at the time of writing. Run 8 is discarded as it is missing in some subjects.
We therefore uses 7 runs of respectively 451, 441, 438, 488, 462, 439 and 542 timeframes and 19 subjects.
The repetition time (TR) is 2s and the spatial resolution is of 1~mm.

Raiders and Gallant dataset pertains to the Individual Brain Charting dataset. These data were acquired using a 3T scanner and resampled to an isotropic spatial resolution of 3~mm. More information is available in~\cite{ibc}. Gallant dataset is refered to as clips in~\cite{ibc}. Data are available at \url{https://openneuro.org/datasets/ds00268}. Datasets gallant and raiders are preprocessed using FSL \url{http://fsl.fmrib.ox.ac.uk/fsl} using  slice  time  correction,  spatial  realignment,  co-registration  to  the  T1  image  and  affine  transformation  of  the  functional  volumes  to  a  template  brain (MNI). The repetition time (TR) is 2s and the spatial resolution is of 3~mm.
The Raiders dataset uses 9 runs of respectively 374, 297, 314, 379, 347, 346,
350, 353 and 211 timeframes. The Gallant dataset uses 17 runs of 325 timeframes each. 
The protocol used for Raiders is the same as the one used in~\cite{haxby2011common} and the protocol used for Gallant is the same as the one used in~\cite{nishimoto2011reconstructing}.

A brief summary of the characteristics of the datasets is available in Table~\ref{tab:dataset}
\begin{table}
    \centering
    \begin{tabular}{c|c|c|c|c}
        \textbf{Dataset} & \textbf{Duration} & $m$& \textbf{  Description} \\
        \hline
         Sherlock & 50 min & 16 & Movie watching (BBC TV show "Sherlock") \\
         Forrest & 110 min & 19 & Auditory version "Forrest Gump" \\
         Gallant & 130 min & 12 & various short video clips \\
         Raiders & 110 min & 11 & Movie watching ("Raiders of the lost ark")
    \end{tabular}
    \caption{Information about datasets (name, duration, number of subjects $m$ and short description)}
    \label{tab:dataset}
\end{table}

All datasets used in MEG have dimensionality 102 since we only consider the
magnetometers. The temporal resolution is 1~ms.

The \emph{CamCAN} dataset~\cite{taylor2017cambridge}
contains the MEG data of 496 different
subjects exposed to an audio-visual stimuli. More precisely, subjects are presented simultaneously an
auditory stimuli lasting 300ms at frequency 300, 600 or 1200 Hz  and a
checkerboard pattern lasting 34ms. 120 trials are available.
The protocol used in the CamCAN MEG dataset is described
in~\cite{taylor2017cambridge}.

\section{Reconstructing the BOLD signal of missing subjects}
\label{app:table}
We report in Table~\ref{reconstruction:table} the R2 score obtained with MVICA, ShICA-J
and ShICA-ML with 20 components as well as a 95\% confidence interval on the experiment
``Reconstructing the fMRI data of left-out subjects''. These data are already
reported in Figure~\ref{fig:reconstruction} but are given here in form of a table.

\begin{table}
  \centering
  \begin{tabular}{c | c | c | c}
    \textbf{Dataset}  & \textbf{Method}   & \textbf{$R^2$ score} & \textbf{Confidence interval}   \\  
    \hline
    forrest  & ShICA-ML & 0.200    & [0.187, 0.213]        \\  
    & ShICA-J  & 0.171    & [0.157, 0.185]          \\
    & MVICA    & 0.191    & [0.177, 0.204]          \\
    gallant  & ShICA-ML & 0.121    & [0.107, 0.135]   \\       
    & ShICA-J  & 0.110    & [0.095, 0.125]          \\
    & MVICA    & 0.114    & [0.099, 0.128]          \\
    raiders  & ShICA-ML & 0.158    & [0.142, 0.174]   \\       
    & ShICA-J  & 0.146    & [0.129, 0.162]          \\
    & MVICA    & 0.144    & [0.124, 0.164]          \\
    sherlock & ShICA-ML & 0.174    & [0.157, 0.191]   \\       
    & ShICA-J  & 0.165    & [0.146, 0.183]          \\
    & MVICA    & 0.161    & [0.142, 0.180]   \\
   \end{tabular}
   \caption{\textbf{Reconstructing the BOLD signal of
       missing subjects}. Median $R^2$ score and 95\% confidence interval.
   }
   \label{reconstruction:table}
\end{table}

\section{Additional experiments}
\subsection{Separation performance}
\label{app:separation}
\subsubsection{Separation performance in function of non-Gaussianity}
We generate data according to model~\eqref{eq:model}. Components $\sbb$ are generated using $s_j = d(x)$ with $d(x) = x |x|^{\alpha - 1}$ and $x \sim \mathcal{N}(0, 1)$.   
Mixing matrices $A_i$ are generated by sampling their coefficients from a standardized Gaussian law. The number of samples is fixed to $n=10^5$ and we vary $\alpha$ between $0.8$ and $1.2$. Each experiment is repeated 40 times using different seeds in the random number generator. We use $p=4$ components and $m=5$ views. We display in Fig~\ref{exp:separatingpower} the mean Amari distance across subjects. The experiment is repeated $100$ times using different seeds. We report the median result and error bars  represent the first and last deciles.
 When $\alpha$ is close to 1 (components are almost Gaussian), ShICA-J, ShICA-ML and multiset CCA can separate components well (but multiset CCA reaches higher amari distance than ShICA). In this regime, MVICA yields much higher amari distance than ShICA-J, ShICA-ML or Multiset CCA but is still better than CanICA which cannot separate components at all.
 As non-Gaussianity ($\alpha$) increases, ICA based methods yield better results but ShICA-ML yields uniformly lower amari distance.
\begin{figure}
\centering
  \includegraphics[width=0.8\textwidth]{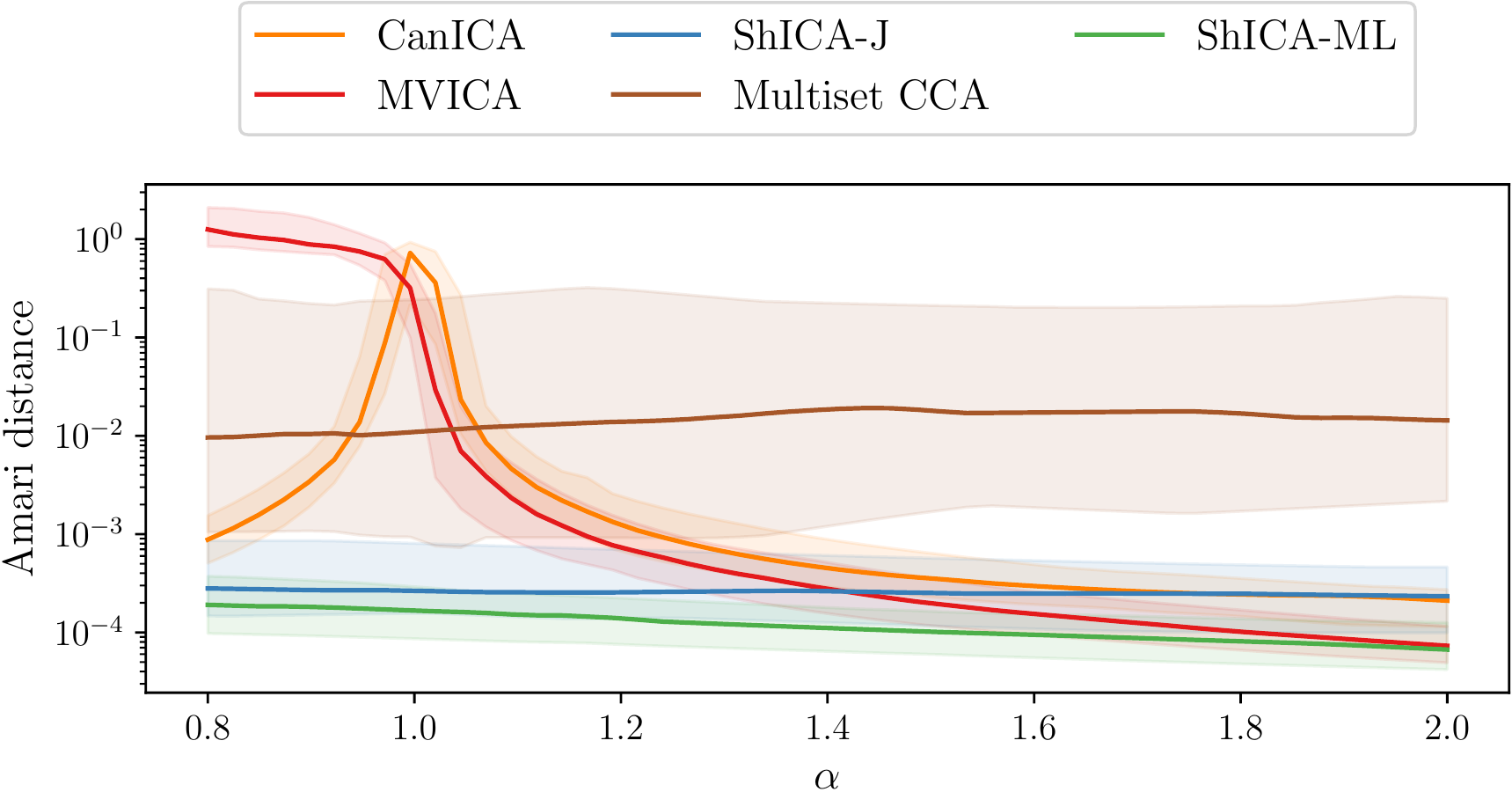}
  \caption{\textbf{ Separation performance in function of non-Gaussianity} Separation performance of algorithms for sub-Gaussian $\alpha < 1$ and super-Gaussian $\alpha > 1$ components}
  \label{exp:separatingpower}
\end{figure}


\subsection{fMRI timesegment matching experiment}
\label{app:timesegment}
We benchmark ShICA on four different real fMRI datasets via a timesegment
matching experiment similar to the one in~\cite{chen2015reduced}. We use full
brain data. The datasets and the preprocessing pipeline are described in Appendix~\ref{app:preprocessing}.
We split the data into a train and test set and algorithms are fitted on the train set.
On the test set, we estimate the shared components from all subjects but one and select a target timesegment containing $9$ consecutive samples in the shared components. We try to localize this timesegment from the components of the left-out subject using a maximum correlation classifier (all possible windows of $9$ consecutive timeframes are considered in the left-out subject excluding the ones partially overlapping with the correct timesegment).
The left panel in Fig~\ref{exp:timesegment} shows that ShICA-ML, MVICA and ShICA-J yield almost equal accuracy and outperform other methods by a large margin. The right panel in Fig~\ref{exp:timesegment} shows that ShICA-J is much faster to fit than MVICA or ShICA-ML.

We would like to highlight here that these experiments are not exactly the same as in~\cite{chen2015reduced} as we use full brain data and they use regions of interest. The code used for this experiment is very similar to the tutorial in \url{https://brainiak.org/tutorials/11-SRM/}. We use the SRM implementation in Brainiak~\cite{kumar2020brainiak}. Also note that the Raiders dataset is different from the one used in~\cite{chen2015reduced} as it involves different subjects and data were acquired in a different neuro-imaging center.

\begin{figure}
    \centering
    \includegraphics[width=0.45\textwidth]{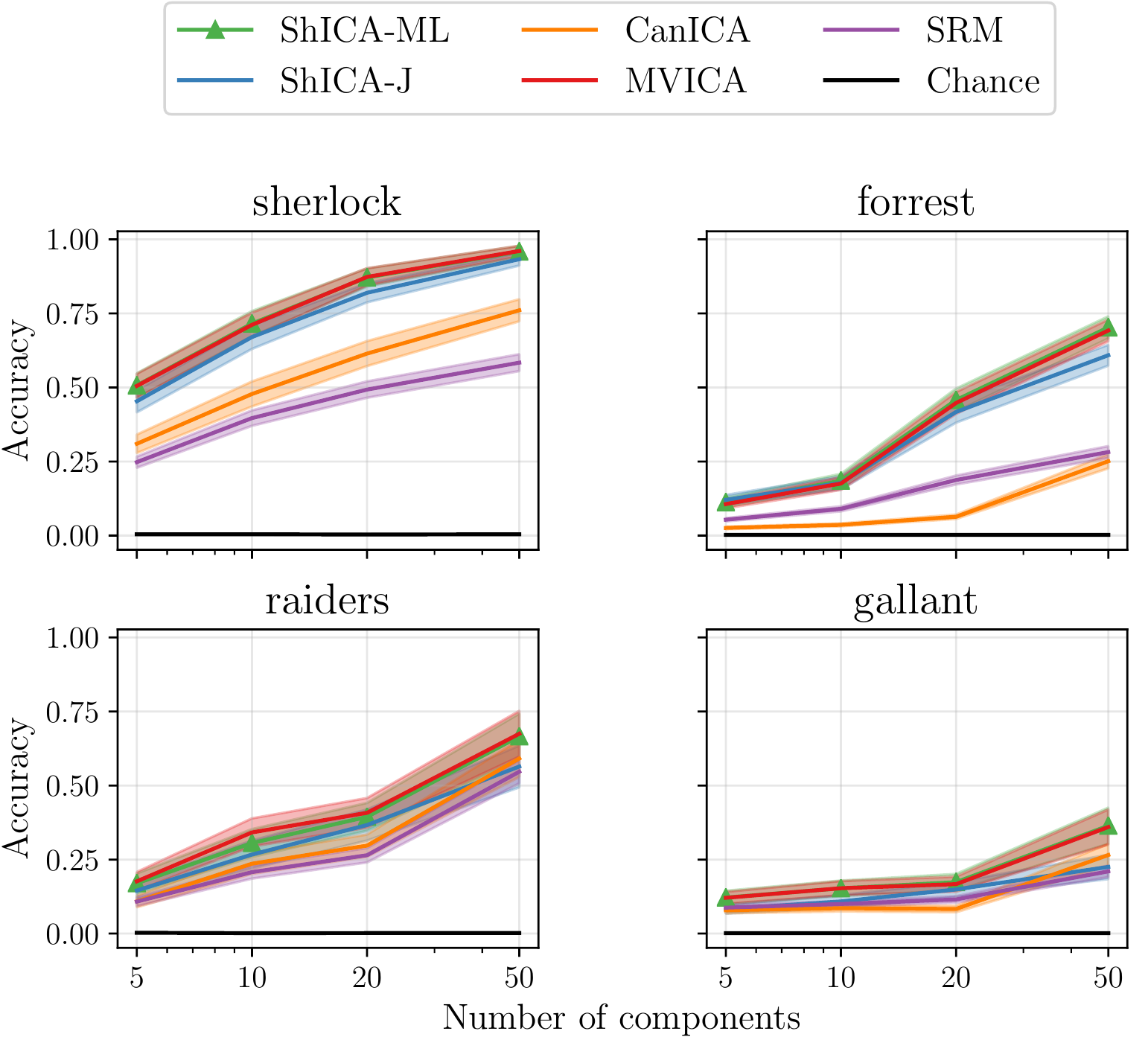}
    \includegraphics[width=0.45\textwidth]{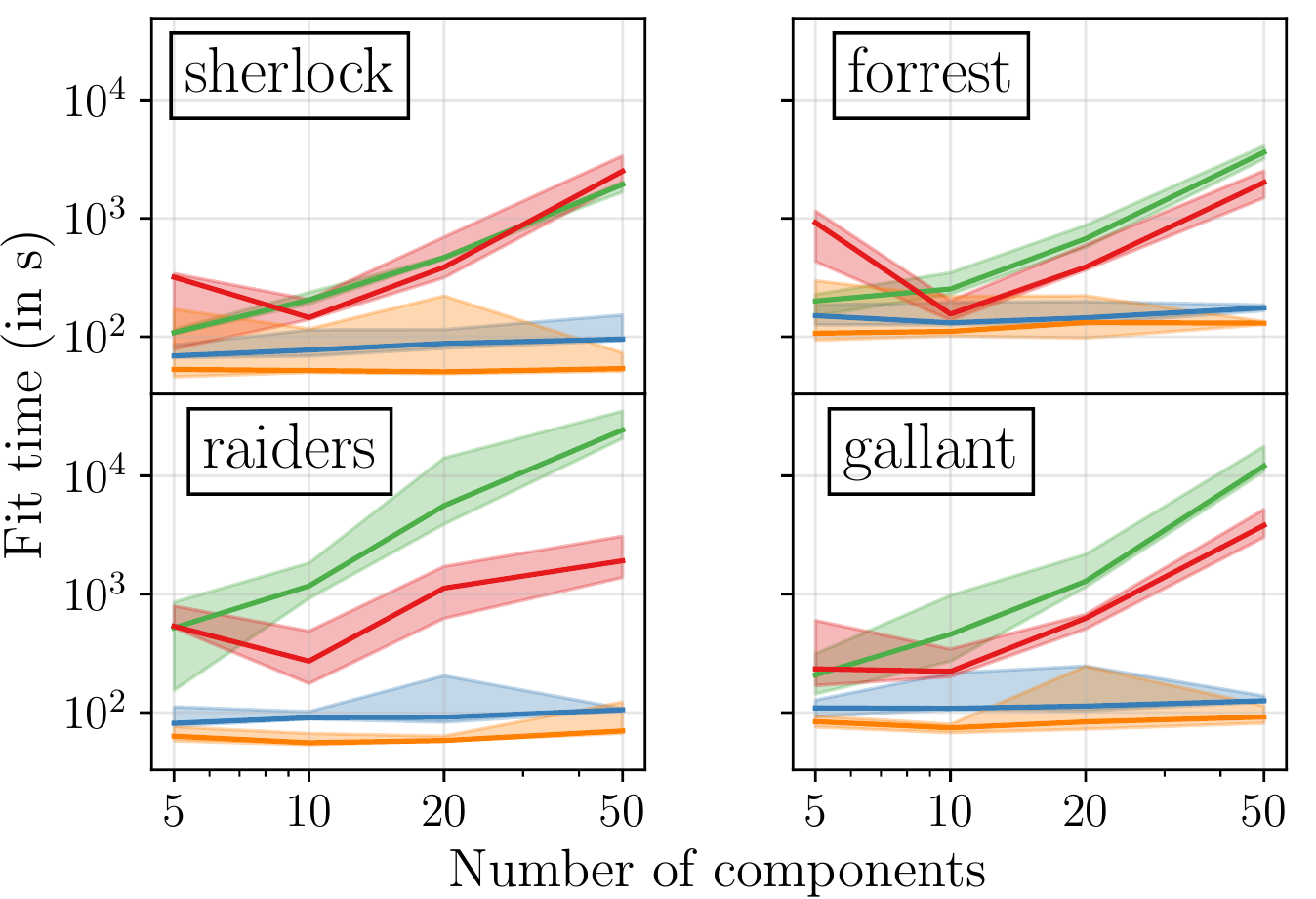}
    \caption{\textbf{Timesegment matching experiment}: (left) Accuracy (right) Fitting time (in seconds)}
    \label{exp:timesegment}
\end{figure}

\subsection{MEG Phantom experiment}
\label{app:phantom}
\subsubsection{Phantom Elektra}
    Dipoles in $m=32$ various locations are emitting the same signal.
    Signal magnitude can be either very high, high or low, leading to 3 datasets: a very clean one, a clean one and a noisy one. These datasets are available as part of the Brainstorm application~\cite{tadel2011brainstorm}. We preprocess the data using Maxwell filtering and low-pass filtering as done in the MNE tutorial \url{https://mne.tools/0.17/auto_tutorials/plot_brainstorm_phantom_elekta.html} and only consider data recorded by the magnetometers.
    We use the very clean dataset to recover the true signal by PCA with 1
    component. Then we reduce the noisy dataset by applying view-specific PCA with $k=20$ components and algorithms are applied on the reduced data. We select the component that is closer to the true one and compute the L2 norm between the predicted component and the true one after normalization.
    Then we attempt to recover the position of each dipole by performing dipole fitting on the mixing operator of each view (using only the column corresponding to the true component). The localization error is defined as the mean l2 distance between the true localization and the predicted localization where the mean is computed across dipoles. 
    Each epoch corresponds to 301 samples and 20 epochs are available in total. We vary the number of epochs between 2 and 18 and display in Fig~\ref{exp:meg_phantom} the reconstruction error and the localization error in function of the number of epochs used.
    ShICA-ML outperforms other methods. ShICA-J gives satisfying results while being much faster.
    
    \subsubsection{Phantom Sinusoidal components}
    For completeness, we display the results obtained on another MEG dataset where the true component is a known sinusoidal and $m=8$ different locations are considered for the dipoles. We vary the number of epochs between 2 and 16 and display in Fig~\ref{exp:meg_phantom_neurips} the reconstruction error and the localization error as a function of the number of epochs used. ShICA-ML outperforms other methods. ShICA-J gives satisfying results while being much faster.
    
\begin{figure}
\centering
  \includegraphics[width=0.7\textwidth]{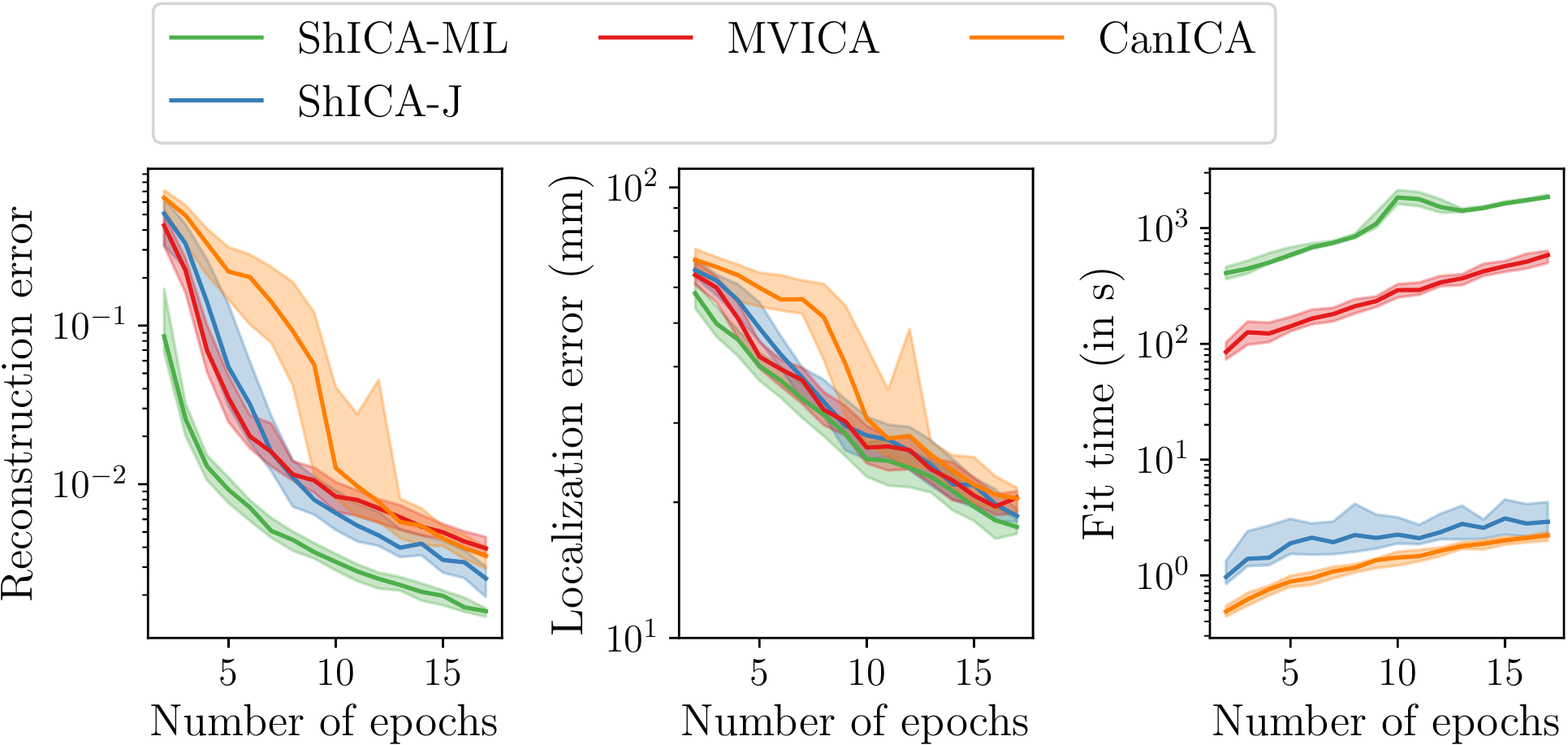}
  \caption{\textbf{MEG Phantom (Elektra)}: (left) L2 distance between the predicted and actual component (middle) Mean error (in mm) between predicted and actual dipoles localization (right) Fitting time (in seconds)}
\label{exp:meg_phantom}
\end{figure}

\begin{figure}
\centering
  \includegraphics[width=0.7\textwidth]{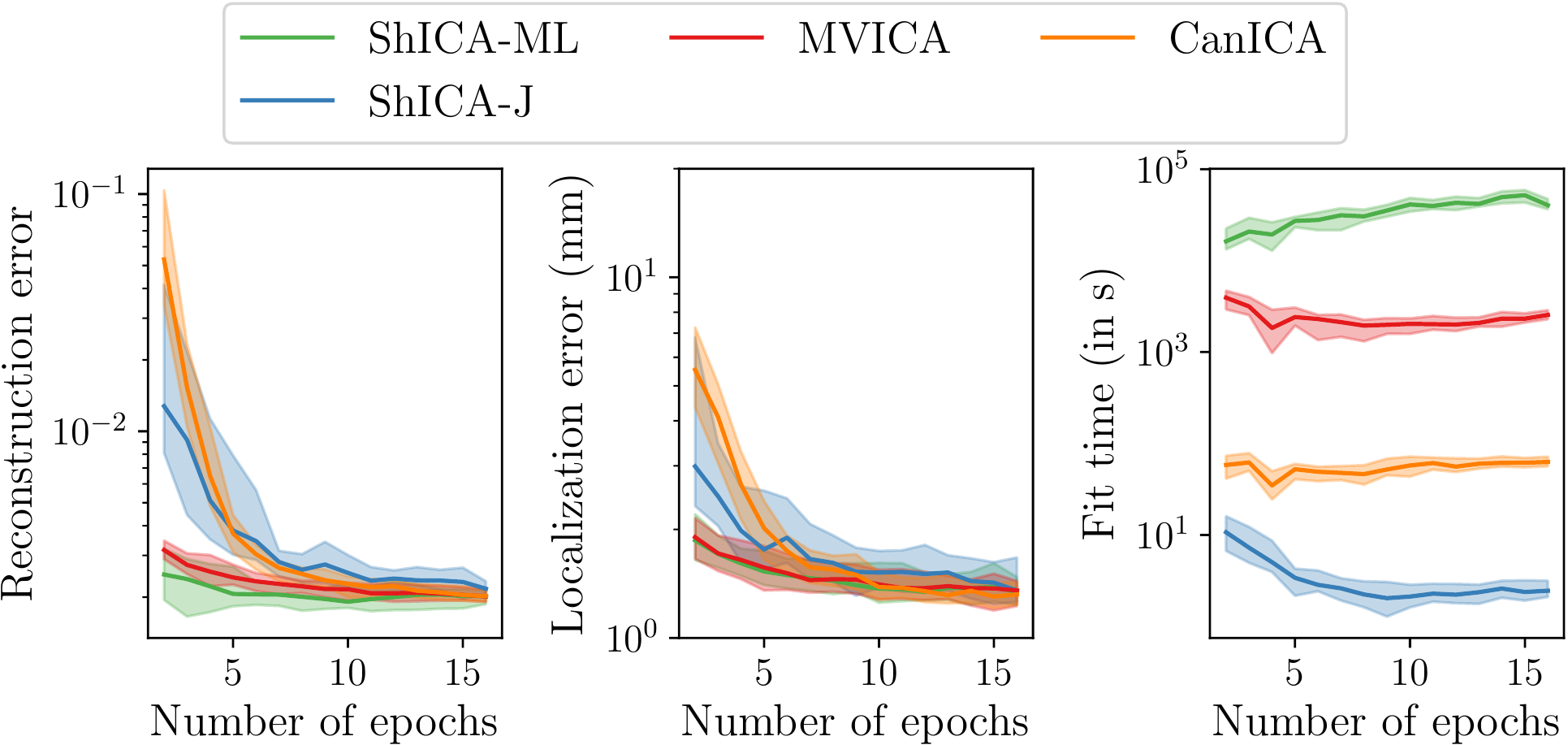}
  \caption{\textbf{MEG Phantom Sinusoidal components}: (left) L2 distance between the predicted and actual component (middle) Mean error (in mm) between predicted and actual dipoles localization (right) Fitting time (in seconds)}
\label{exp:meg_phantom_neurips}
\end{figure}

\subsection{CamCAN MEG components}
We consider the CamCAN dataset used to produce Fig~4. We use $m=496$ subjects and fit ShICA-ML with $p=10$ components. We localize the components of each subject using sLoreta~\cite{pascual2002standardized}. Then components are registered to a common brain and averaged. Thresholded maps are displayed below along with the time courses of each component. Components obtained with ShICA-ML highlight the ventral visual cortex and auditive cortex. The results suggest that the response of the auditive cortex is faster and lasts a shorter time than the response of the ventral visual cortex.

{\centering
\includegraphics[width=0.4\textwidth]{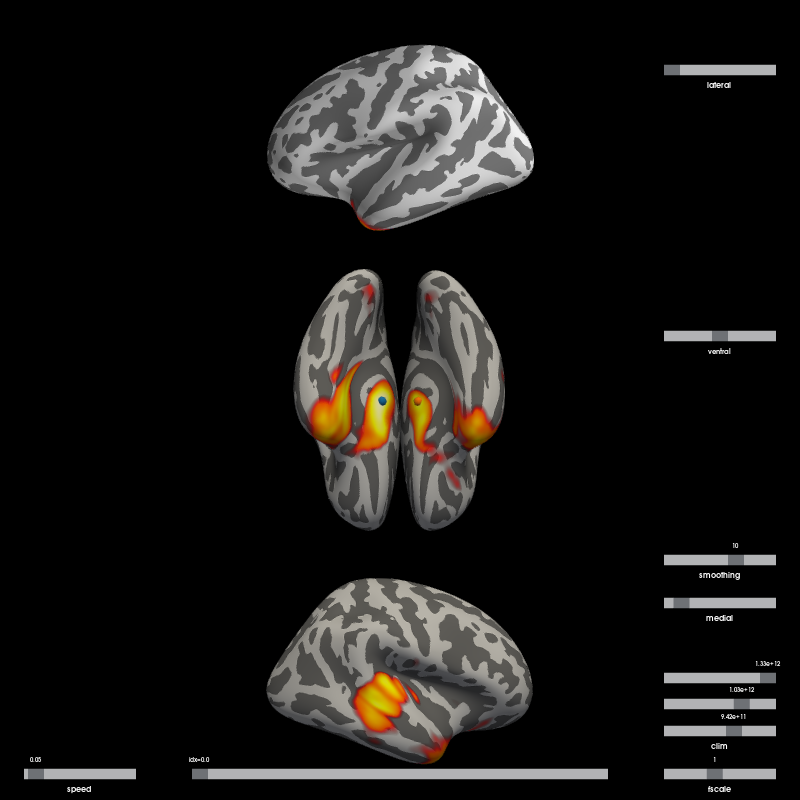}\includegraphics[width=0.4\textwidth]{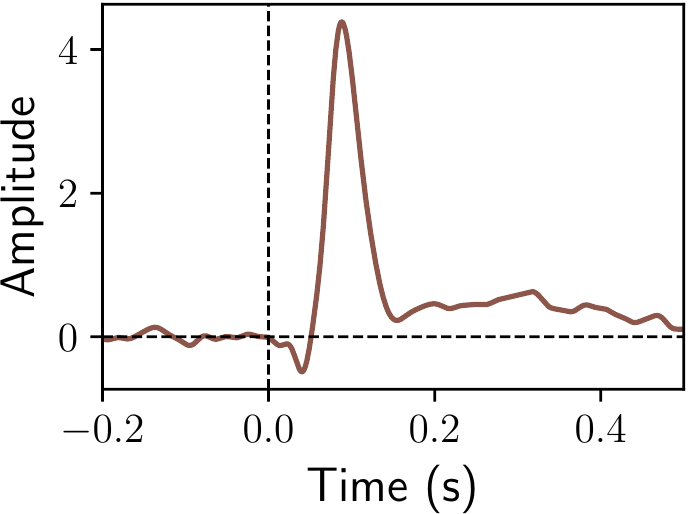} \\
\includegraphics[width=0.4\textwidth]{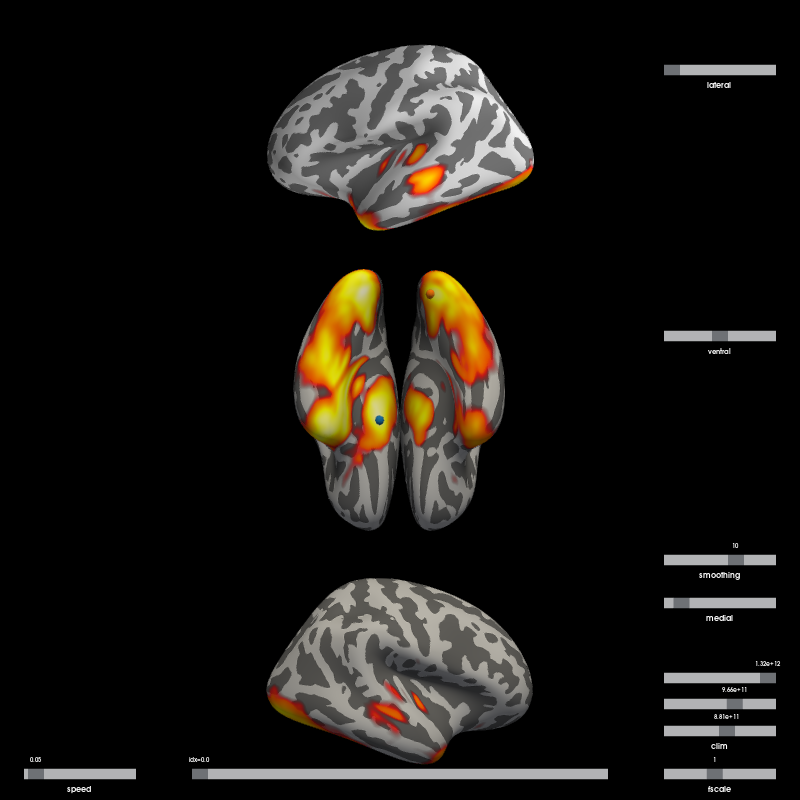}\includegraphics[width=0.4\textwidth]{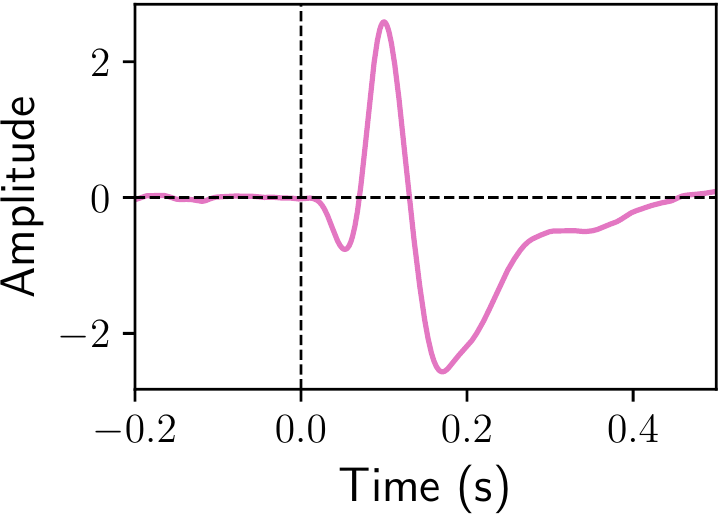} \\
\includegraphics[width=0.4\textwidth]{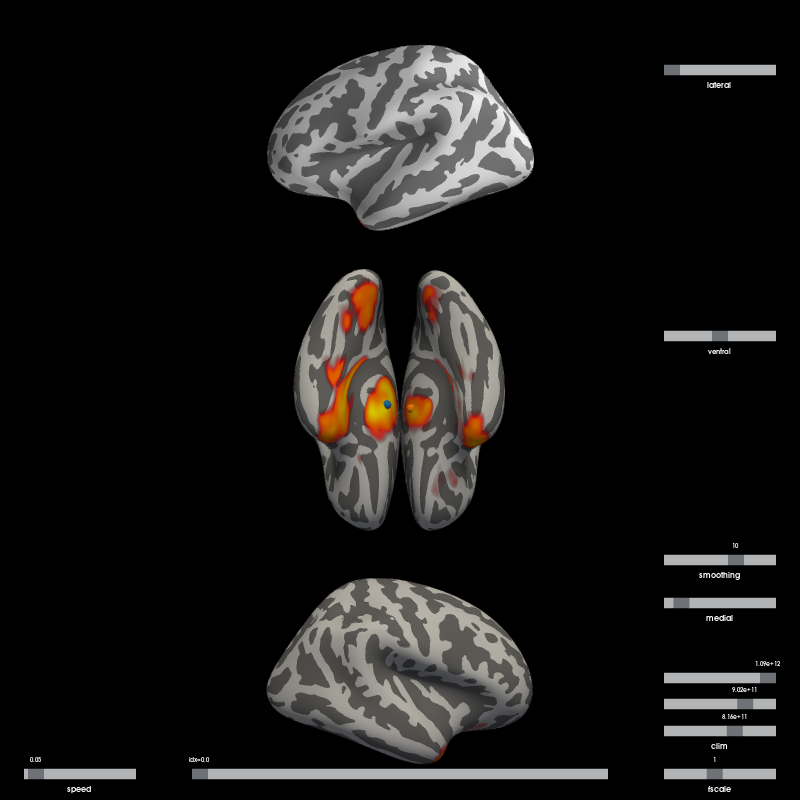}\includegraphics[width=0.4\textwidth]{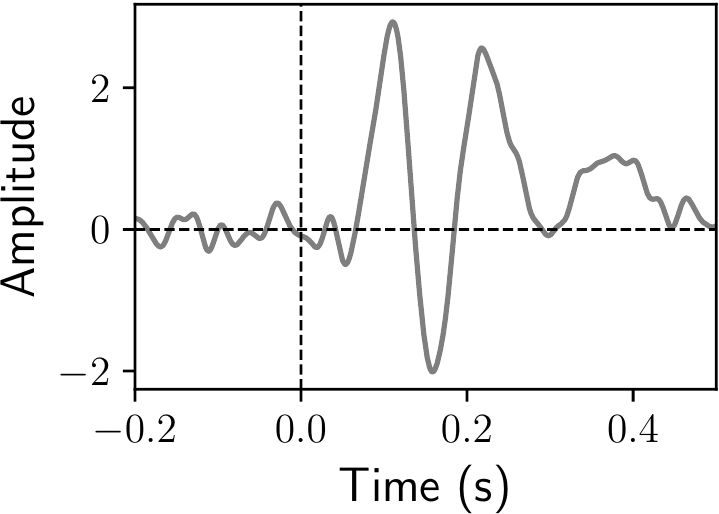} \\
\includegraphics[width=0.4\textwidth]{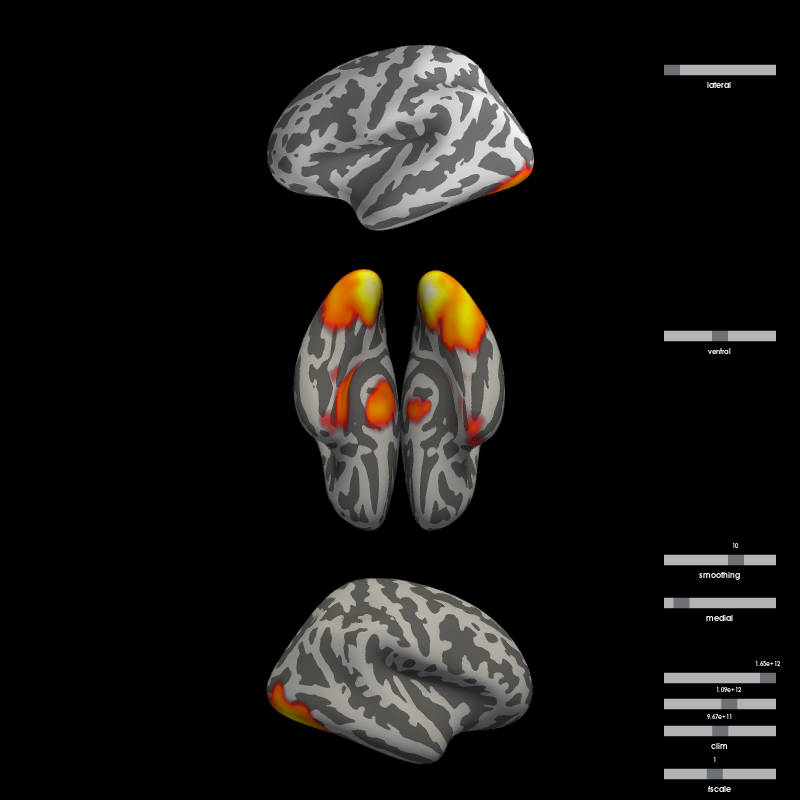}\includegraphics[width=0.4\textwidth]{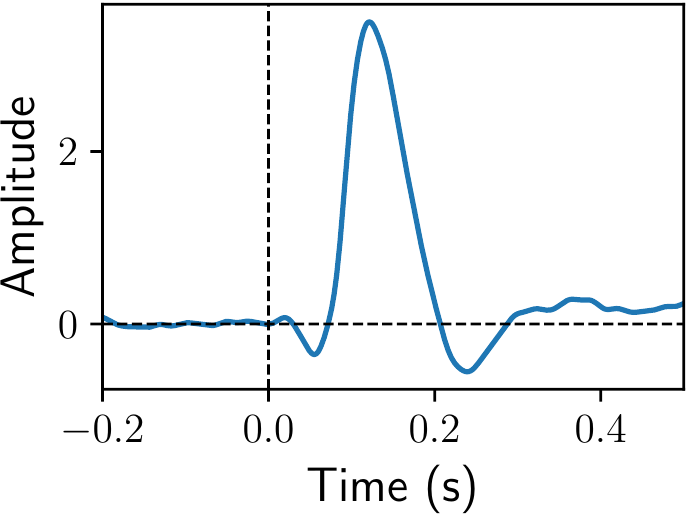} \\
\includegraphics[width=0.4\textwidth]{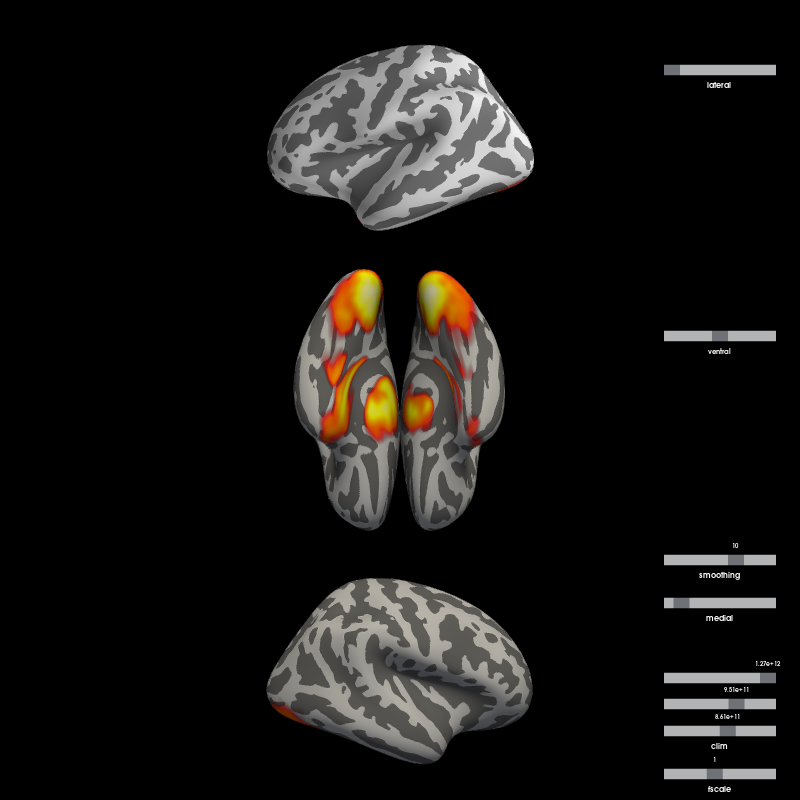}\includegraphics[width=0.4\textwidth]{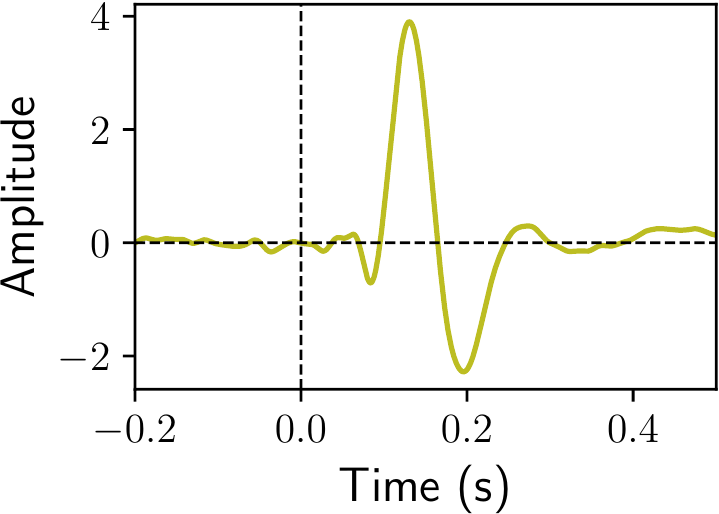} \\
\includegraphics[width=0.4\textwidth]{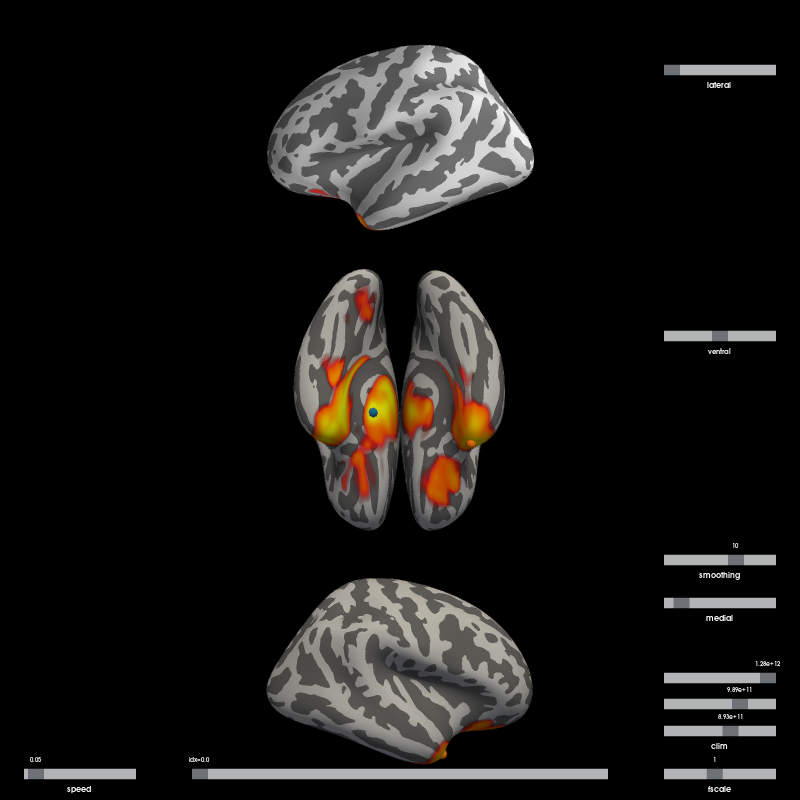}\includegraphics[width=0.4\textwidth]{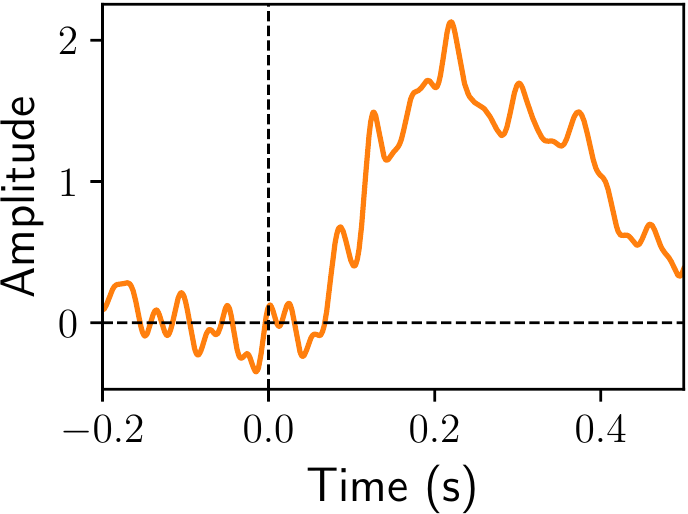} \\
\includegraphics[width=0.4\textwidth]{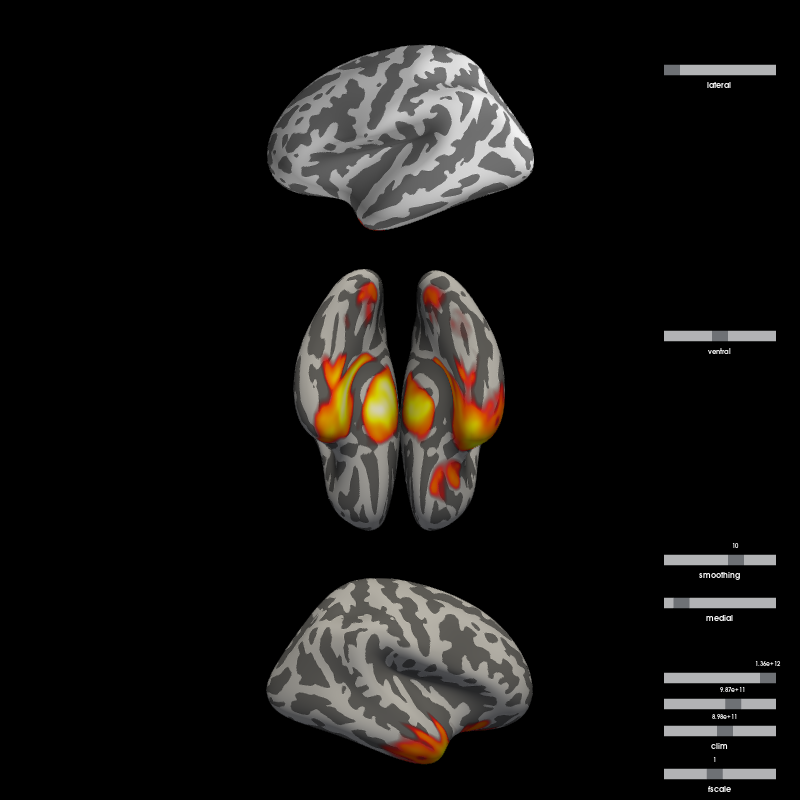}\includegraphics[width=0.4\textwidth]{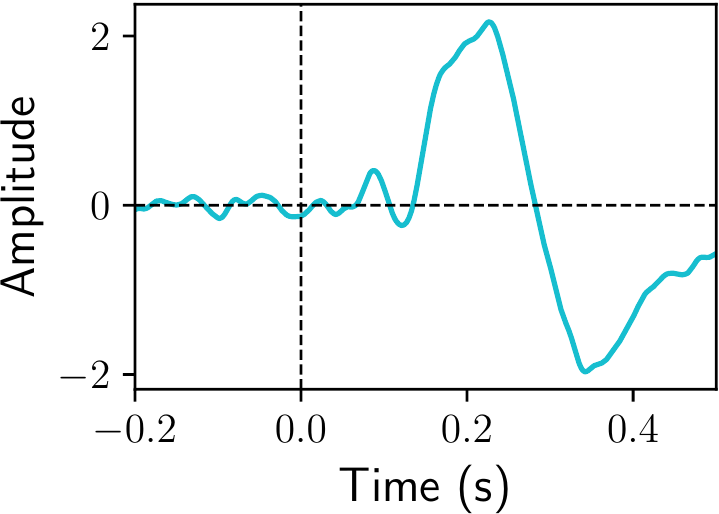} \\
\includegraphics[width=0.4\textwidth]{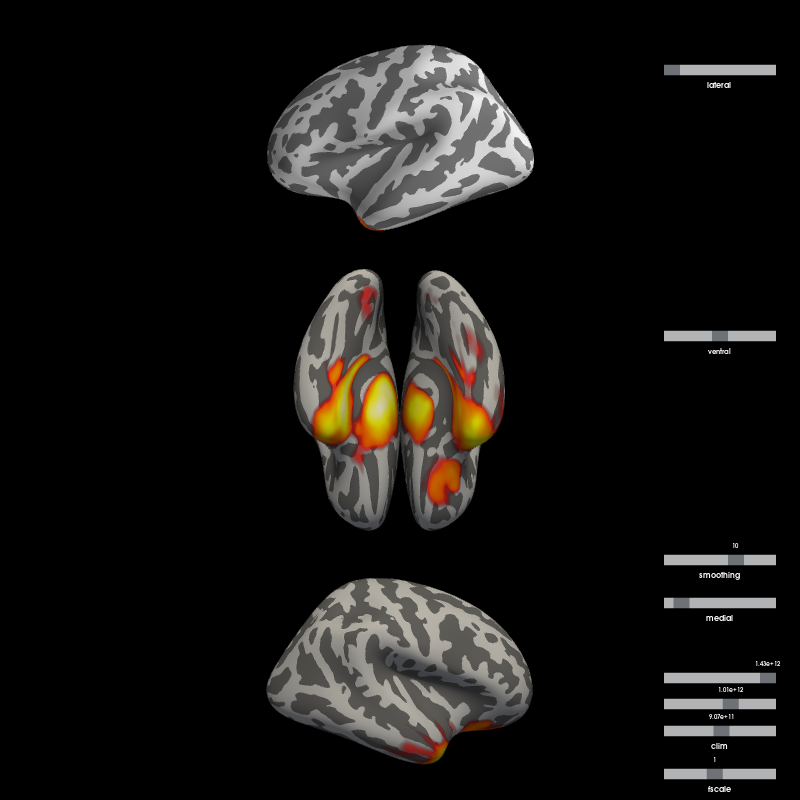}\includegraphics[width=0.4\textwidth]{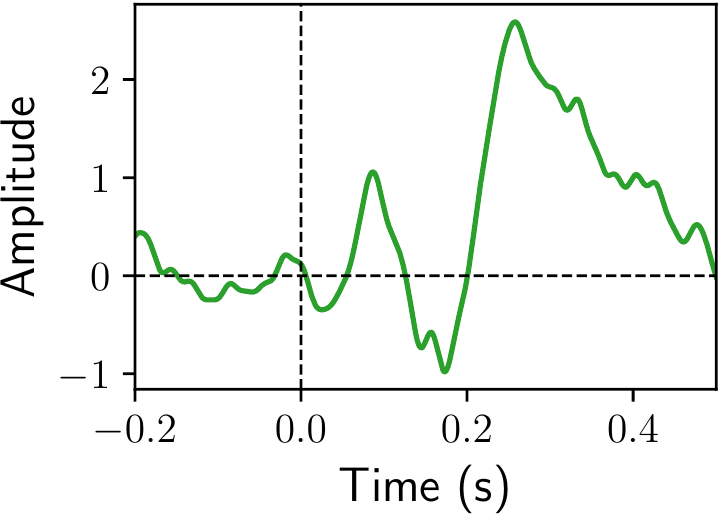} \\
\includegraphics[width=0.4\textwidth]{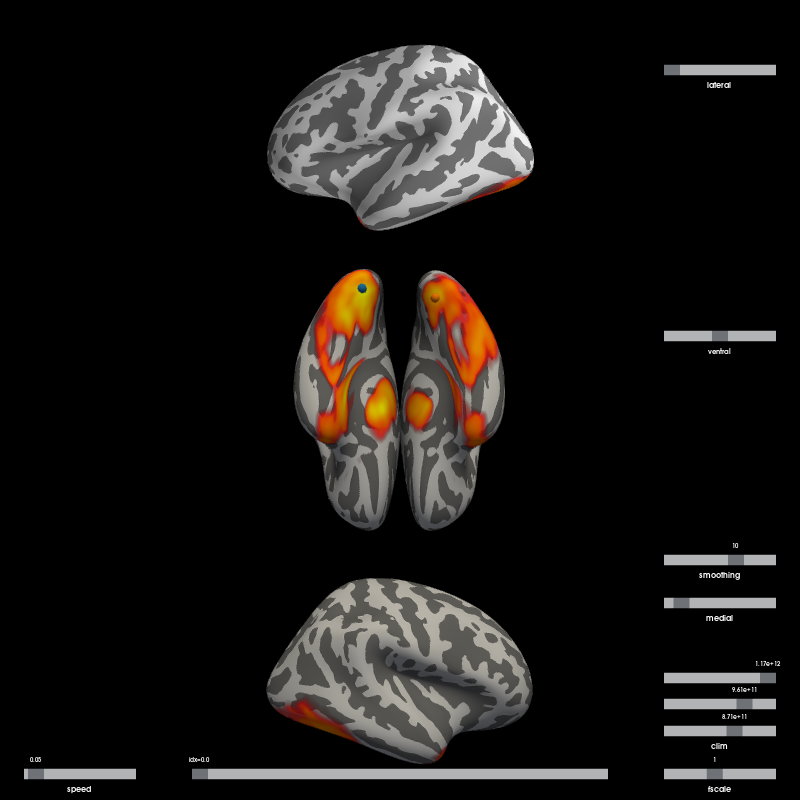}\includegraphics[width=0.4\textwidth]{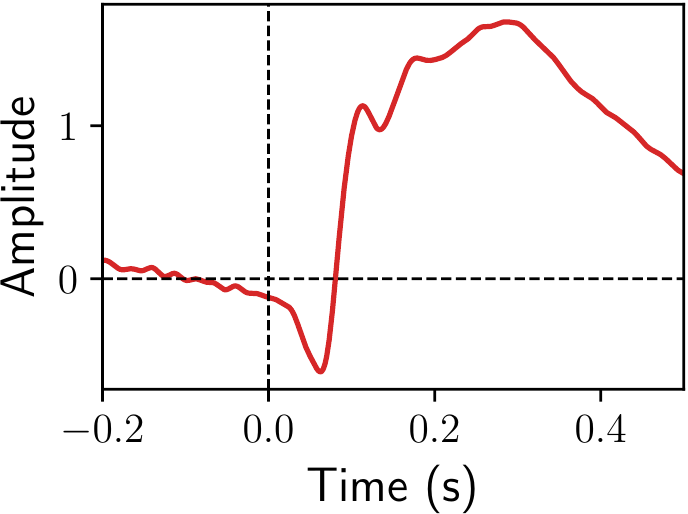} \\
\includegraphics[width=0.4\textwidth]{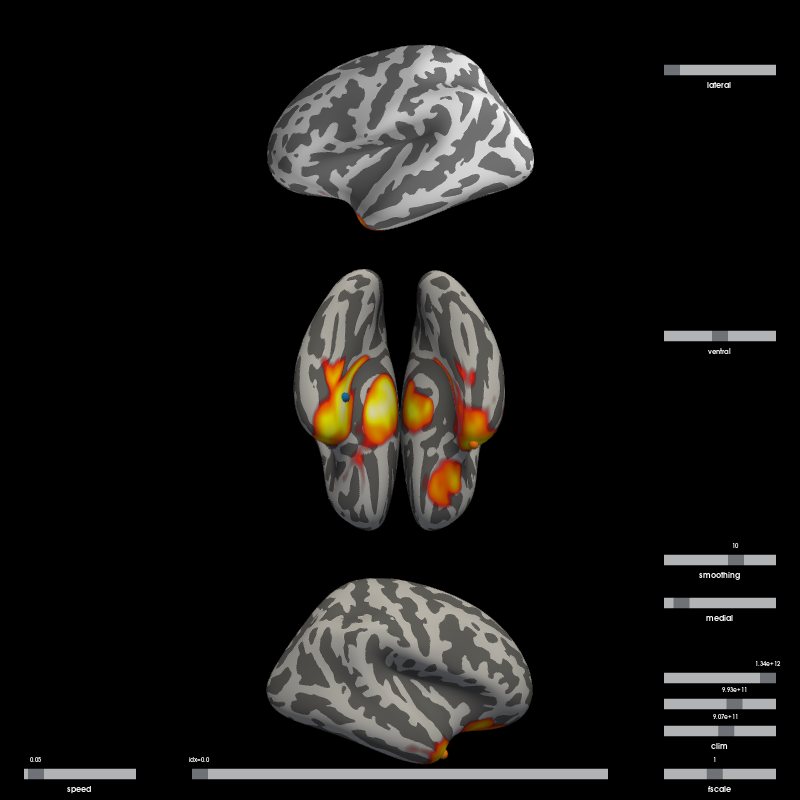}\includegraphics[width=0.4\textwidth]{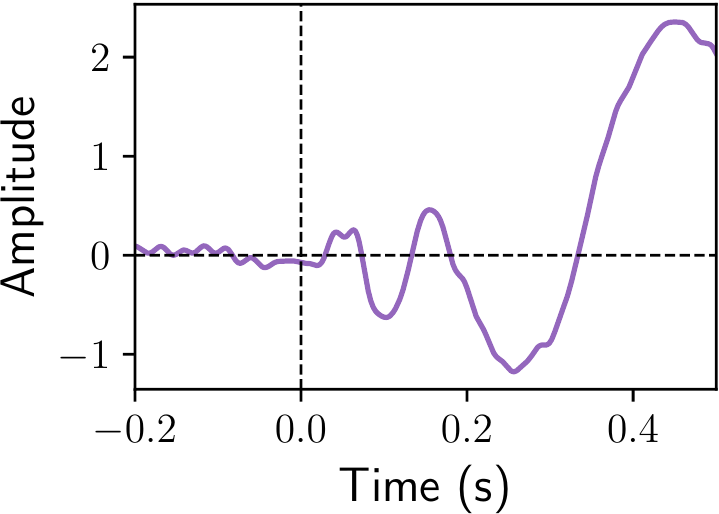} \\
}

\bibliographystyleonline{plain}
\bibliographyonline{biblio}

\end{document}